\documentclass[11pt,a4paper]{article}
\usepackage[a4paper,textwidth=183mm,top=20mm,bottom=22mm]{geometry}
\usepackage{fontspec}
\usepackage{newunicodechar}
\newunicodechar{⁰}{\textsuperscript{0}}
\newunicodechar{⁴}{\textsuperscript{4}}
\newunicodechar{⁵}{\textsuperscript{5}}
\newunicodechar{⁶}{\textsuperscript{6}}
\newunicodechar{⁷}{\textsuperscript{7}}
\newunicodechar{⁸}{\textsuperscript{8}}
\newunicodechar{⁹}{\textsuperscript{9}}
\usepackage{tabularx,booktabs,graphicx,setspace,fvextra,caption,enumitem,float,xurl}
\usepackage[hidelinks]{hyperref}
\newcolumntype{Y}{>{\raggedright\arraybackslash}X}
\newcolumntype{Z}[1]{>{\hsize=#1\hsize\raggedright\arraybackslash}X}
\RecustomVerbatimEnvironment{verbatim}{Verbatim}{breaklines=true,fontsize=\footnotesize}
\setlist[enumerate]{leftmargin=2em}

\AtBeginDocument{\typeout{DISPLAY_LINEWIDTH=\the\linewidth}}
\newcommand{\numkeep}[1]{\mbox{#1}}
\begin{document}
\begin{center}{\Large\bfseries Curved Knowledge in Flat Models: Building Expert AI at Scale}\\[6pt]
Kwan Soo Shin*\\ComplexSystems Research, Asan, Republic of Korea\end{center}

\noindent\textbf{Summary.} Intelligence that must act before outside correction arrives must carry the structure of its expert knowledge, and hyperbolic geometry represents hierarchy with a fidelity no Euclidean space of comparable dimension matches\textsuperscript{1,2}. Yet hyperbolic language models confine that geometry to adapters\textsuperscript{3} or rebuild the transformer\textsuperscript{4}, and an independent audit of released hyperbolic vision language checkpoints finds their hierarchy mechanisms inactive\textsuperscript{5}, which has been read as evidence about the geometry or its objective. Here we show that it is also evidence about arithmetic. The constraint defining the manifold is a difference of two terms that grow together; for the unclamped map at unit curvature and temperature it fails in every format we measured, double precision included, at the deep-layer norms of the language models measured here, beyond the limit previously derived\textsuperscript{6}. A common remedy, a norm bound\textsuperscript{7}, leaves the executed map no radial sensitivity where it acts. In a prespecified test, a tighter bound cost a Lorentz system more of the held out structure it started with than a flat one. Structural supervision has run in six expert systems whose retained input archives and source pools hold at least 13,787,131 records. In paired runs with a published recipe's objective held fixed, removing its starting scale remedy prevented every update and removing its single precision regions raised the radius reached early in training. Arithmetic bounds the computable region, and within it, alongside the objective, arithmetic and its remedies select which runs leave a record and the states they reach on the way, so a flat surviving checkpoint cannot by itself establish that the geometry is inert.

\subsection{1. Intelligence that must travel}\label{intelligence-that-must-travel}

Intelligence acting beyond the reach of external correction must carry the structure of its expert knowledge, and hyperbolic space represents hierarchy with distortion no Euclidean space of comparable dimension matches\textsuperscript{1,2}, yet for the unclamped map at unit curvature and temperature the constraint defining that space fails in every format we measured, double precision included, at the deep-layer norms of the language models measured here. Hyperbolic language models keep the geometry small\textsuperscript{3,8,9,10} or rebuild the transformer\textsuperscript{4}, and an independent audit of released hyperbolic vision language checkpoints finds their hierarchy mechanisms idle\textsuperscript{5}, findings read as evidence about the geometry or its objective. Here we show that they are also evidence about arithmetic. Earlier and concurrent work derived and repaired the limit (section 2.1), the concurrent work also reporting that the ambient distance fails or reports false convergence in its layer expressivity test (Supplementary Note M14); what it does not show is that the arithmetic and the remedies that make the geometry computable, alongside the objective, select which runs leave the record from which the field inferred what learned geometry can do, and the states those runs reach on the way (section 2.3). With a published recipe's objective held fixed, removing its starting scale remedy prevented every update in paired runs, and removing its single precision regions raised the radius reached early in training. A test fixed before its runs did not separate the end states of Lorentz and flat systems at the higher clamp bound, while the tighter clamp changed how they compare, costing the Lorentz system more of the held out structure it started with than the flat one (section 3.3).

Humanity is carrying intelligence beyond the reach of the infrastructure that made it powerful: craft on interplanetary transfer, vehicles beneath the ocean\textsuperscript{11}, machines underground\textsuperscript{12}, teams inside a collapsed disaster zone and systems under denied communication\textsuperscript{13} share one condition, that external correction arrives later than the action it was meant to correct (Methods M15). A system in that condition cannot answer a hard question by calling for more scale. It must already contain what it needs, and what it contains must remain available in the arithmetic that executes it. A 2025 assessment by the agency that operates the longest running crewed missions rates the risk that onboard capability is insufficient once ground support recedes as red for extended lunar surface and Mars missions\textsuperscript{14}. A resident expert system cannot buy competence with parameters it cannot carry\textsuperscript{15,16,17}, so the structure of its domain, which concepts belong together, which judgments govern a method and which relations connect separate domains, has to become a computational resource in its own right\textsuperscript{18}. Expert knowledge is organised in kinds and levels\textsuperscript{19,20}, and against Euclidean space the advantage of hyperbolic space is not a number but a theorem, since no Euclidean dimension at all achieves comparable distortion on trees\textsuperscript{2}.

The one architecture that escapes the restrictions that keep the geometry small names the difficulty as the ``poor scalability of existing hyperbolic LMs''\textsuperscript{4}; the instability that accompanies hyperbolic results has been reported without being used as evidence about the geometry\textsuperscript{5}, although which research ideas succeed has been argued to depend on the hardware and software available to test them\textsuperscript{21}.
The constraint that defines the hyperbolic manifold is written as a difference of two terms that grow together, and a quantity of that form dies with magnitude: the two coordinates of a projected point round to one stored number from an argument of 2.8 in bfloat16, the format chosen to make scale cheap, 8.4 in single precision and 18.7 in double precision. At unit curvature and temperature, standard pretrained models deliver arguments past that point at the depths where geometric supervision is inserted, and the remedy that lets a method continue there, the hard norm clamp, sets the radial singular value of the executed map to exactly zero (Fig. 1e), so the map blocks the direct radial channel for hierarchical depth while retaining angular sensitivity. Both are measured before anything is trained. The remedy's further costs are measured by readout interventions and training comparisons: setting every radius of a trained head to one value, which is what a clamp acting on every node computes, raised the mean rank of held out ancestors from 2.86 to 5.55, and on the adapter a clamp of 2.0 left final update norms 12.2 times those under 5.0 in geometric mean (sections 2.6 and 2.7). In paired runs of the published code of section 2.3, with the objective held fixed, removing the starting scale remedy, the recipe's learnable factor that brings each embedding to unit expected norm at initialisation, prevented every optimiser update in 10,000 attempted iterations at each of three seeds, and removing the single precision regions, in which the recipe computes the map and the loss in single precision inside half precision training, moved the image radius reached after 10,000 paired iterations from \numkeep{0.254--0.256} to \numkeep{0.411--0.413} (Fig. 1a, b). Raising the clamp limit and calibrating temperature opened an interval where manifold membership, a radial gradient and the measured departure from Euclidean geometry held together at initialisation, and adaptive control held the batch median inside it for 500 steps in each of six seeds, at a measured price: the temperature reached its preset floor, and interval occupancy did not transfer to a separate probe (section 5). With the objective's targets rescaled before the runs, the calibrated temperature held fixed kept every traced batch inside that interval at all 500 recorded updates, and the separate probe at every scheduled evaluation (section 5). Confining the geometry to the objective leaves no trainable factor between the representation and the clamped map, whose derivative with respect to its own input is bounded by a constant fixed by the clamp; that position was occupied at production scale on standard pretrained models (section 4). The condition that decides what such a limit costs is set out next, and it is a consequence of the result rather than a premise of it (section 2.8).
The number format and the computation path set the arithmetic wall; placement determines which trainable paths encounter the map and its remedy. We show where the identity fails, why the field never sat there, and what the remedy costs. The complete account of each section and of the Methods is given in the Supplementary Information under the same numbers.

\subsection{2. A numerical wall inside standard models}\label{a-numerical-wall-inside-standard-models}

\subsubsection{2.1 The limit, before any training}\label{the-limit-before-any-training}

In the parametrisation examined here, a representation is mapped into hyperbolic space by the exponential map\textsuperscript{22}, whose Lorentz form returns two coordinates, the hyperbolic cosine and the hyperbolic sine of the representation norm scaled by the square root of curvature. Both grow as one half of the natural exponential, and the whole of the difficulty follows from that one fact. The constraint that defines the manifold is the difference of their squares, which must equal 1, and the pair fails in two ways, in a fixed order (Methods M1). \textbf{The identity fails first.} Once the relative gap between the two coordinates drops below the resolution of the format, they round to the same stored number and the constraint evaluates to 0 instead of 1: at an argument of 5 in bfloat16, the two coordinates are \texttt{74.0} and \texttt{74.0}, bit for bit identical. \textbf{The range fails second}, once the hyperbolic cosine of the argument exceeds the largest representable value, at 89.4160 in single precision and 11.7830 in half precision. The first argument at which the two coordinates round to the same stored number is 8.4 in single precision under the mixed precision convention of Methods M1, 4.2 in half precision, 2.8 in bfloat16 and 18.7 in double precision for correctly rounded coordinates (Extended Data Table 1), the last being the Lorentz capacity that ref. 6 derives analytically in 64 bit arithmetic. Each lies in a band derived, not fitted, from the width of the significand (Extended Data Table 1); concurrent work reaches the lower end of that band from a criterion on distances, 2.8 in bfloat16 and 8.3 in single precision\textsuperscript{23}, and two criteria meeting at one value is evidence that the wall belongs to the format and not to a way of reading it. That representation error in this geometry depends on the precision is established\textsuperscript{24}; what had not been stated is where the two coordinates first become one stored number in the arithmetic path a model actually runs (Methods M1). The failure is a property of the written form and not of the geometry (Extended Data Fig. 1a): factored as a product of exponentials the same constraint loses the mantissa wall and keeps a range wall of its own, and invariants kept as sums of bounded terms, the rotary identity among them, show no degradation at all (Methods M1). The distance the structural objective evaluates is the same difference of co growing terms and inherits the same fate, and it can be written instead as a sum of non negative terms; that rewrite appears in a comparison of hyperbolic representations, in two libraries and in a concurrent polar transformer\textsuperscript{23,25,26,27} and not in the two released Lorentz training implementations whose source we read, which bounded the argument instead (Methods M13).

\subsubsection{2.2 Raw hidden states are already past it}\label{raw-hidden-states-are-already-past-it}

We measured the norm of hidden representations in three standard models of 7 to 8 billion parameters at five depths each, on stand in prompts written in the register of the six project corpora, at the last non padding token of every sequence because that is the position the projection actually receives, and evaluated the Lorentz identity in each of the four floating point formats (Methods M2, Extended Data Table 2; Extended Data Table 3 shows every token position, as Extended Data Fig. 1b does for the first two models). Divided by the square root of the hidden width, the scale at which a fixed per coordinate magnitude gives a fixed argument, the deepest medians are arguments of 9.1417, 0.8769 and 1.1566: every vector passes the stored coordinate membership test in double precision, in single precision all but the 36 at the deepest layer of the first model do, and stored in bfloat16 the first model loses membership from its 14th layer onward, 22, 3 and 0 of 36 passing there, while the deepest layers of the other two keep 33 and 27 of 36 (Methods M2). The figures that follow take the norm itself as the argument, at unit curvature and unit temperature: the unscaled reference condition, which bounds the argument from above for any temperature of at least 1; the implementations examined in section 2.3 each move away from it by one of the retreats listed there (section 2.8). \textbf{The failure persists through double precision at these measured arguments.} In the first model the median norm at the 7th layer is 40.8 and at the 28th is 547.3, and at those arguments the identity fails completely in single precision, in bfloat16 and in half precision, \textbf{and it fails completely in double precision as well}. Each further bit of significand moves that point by only half the natural logarithm of two, about 0.35, so at these arguments the leading budget is about 118 and 1,580 bits, before the further bits a chosen error tolerance requires; higher precision can restore the identity at that growing price, while changing the scale of the argument, the form of the invariant or the place of the map avoids it (Methods M1). In the second and the third the crossing comes later, at the 24th layer, where the medians are 19.0 and 28.1; the third is the base to which the registry of section 4 assigns 47.6\% of its counted updates. At unit curvature and temperature, none of the 252 last token vectors captured at the 7 of the 12 non embedding depths where all 36 values exceed the onset in every format passed the stored coordinate membership test in any of the four formats, double precision included (Methods M2). \textbf{The embedding layer is the single exception}: no representation there reaches a norm of 1.11.

\subsubsection{2.3 The prediction, and an independent audit}\label{the-prediction-and-an-independent-audit}

The arithmetic no longer guarantees the identity of the exponential map beyond a norm that the median representation exceeds at every sampled non embedding depth in all three models. This predicts that a hyperbolic method that survives contact with a real model will be found to have retreated from evaluating the map there, and the four implementations whose code or text we examined agree with it, each bounding the argument before the map is evaluated (Methods M13). The retreat takes one of these positions: stay at the embedding layer, drive curvature low enough that the scaled norm falls back under the limit, or clamp or rescale the representation so that it never reaches the limit. Confining the map to the objective is a choice of the trainable path and does not by itself bound the argument: a map evaluated at the loss on the raw norms of section 2.2 meets the same cancellation. The record says so in its own words: the first hyperbolic vision transformers adopt a curvature of 0.1 below their own estimate because ``smaller values tend to provide better stability''\textsuperscript{7}, and a released vision language recipe computes the ``exponential map and losses for MERU in FP32 precision for numerical stability''\textsuperscript{28} (Methods M14; Extended Data Fig. 2).

The strongest confirmation comes from an audit conducted independently of this work and for a different purpose\textsuperscript{5}. Across seven released checkpoints from three hyperbolic vision language families it finds that \textbf{no released checkpoint enters the regime u greater than 1}. In its separate from scratch current-GRIT reproductions, where half precision Lorentz operations under mixed precision caused numerical instability and contrastive loss collapse, it added ``explicit fp32 casts'' ``to all Lorentz operations'', and it states that ``The numerical failure itself is not treated as evidence for the hierarchy conclusions.''

We think it is evidence. The arithmetic is part of what the audit measured: it fixes the region in which the exact map can be evaluated and still return a point on the manifold, and the near flat regime the audit measures so carefully is where optimisation came to rest inside that region. The audit's own retrained runs ``converge to the same dimensionless radius as the released checkpoints''\textsuperscript{5}, reached from an initialisation the recipe places at a dimensionless radius of 1.175 (Fig. 1d; Methods M13). Our own paired runs of the published code, at the batch and schedule of the audit's retraining and on a public image text corpus, span the first 10,000 iterations, where the audit attributes curvature pressure by loss, and their reference arm reaches the curvature floor when the audit's full objective runs of its larger MERU model do (Methods M17.1). With the objective held fixed, removing the starting scale remedy prevented every optimiser update in 10,000 attempted iterations at each of three seeds (Fig. 1a): nearly every arccosine input of its entailment term reached a bound that single precision stores at ±1, where the derivative is not finite, so the loss scaler skipped every update, while double precision kept the gradient finite from the same outputs; removing the single precision regions instead shifted the image radius reached after 10,000 paired iterations from \numkeep{0.254--0.256} to \numkeep{0.411--0.413} (Fig. 1b) and delayed curvature's arrival within 0.005 of its floor by more than 2,635 iterations; in the audit's runs of larger models, removing the entailment term, the dominant pressure, delays it by 2,583 and 2,733 steps in three seed means. The more accurate arithmetic is the one that reached the flatter state: with the published single precision regions the local distortion factor in the audit's convention was \numkeep{1.0105--1.0108} and without them \numkeep{1.0269--1.0271}, both near the Euclidean limit. The reference runs and those without the single precision regions all completed the 10,000 iterations with text to image recall in overlapping ranges (Fig. 1c), so neither survival nor the task score revealed that the arithmetic had moved the radius. At the released checkpoint itself, read as a diagnosis of the state it is in, the same switch turns the objective's push on curvature from downward to upward through the contrastive term alone, in all 16 held out batches of the recorded measurement, and in their sum under every kernel setting tried (Methods M17.1). Restarted from those weights with a new optimiser state and schedule, the runs with half precision distances showed a transient rise of curvature above 0.105, from iteration 924 to at most 3,175, before it returned to its floor, and a larger image radius, \numkeep{0.455--0.458} against \numkeep{0.245--0.249} after 10,000 paired iterations at all three seeds (Methods M17.1). Arithmetic and its remedies condition, alongside the objective, whether training proceeds and the states it reaches on the way.

\subsubsection{2.4 The coordinate and gradient routes}\label{the-coordinate-and-gradient-routes}

The walls act through two distinct routes, and they are removed by different means. \textbf{The coordinate route.} On two synthetic trees the maximum Lorentz time coordinates are 87,412 and 91,010 against a half precision maximum of 65,504, and retrieval quality under half precision falls to 0.205 and 0.142 when the coordinates are read as the difference of section 2.1, whereas the same coordinates read in the rewritten form retrieve 0.840 and 0.905 (Methods M5). On LucaPhylo, an external phylogenetic model whose third stage was trained by its authors and whose fourth stage we fine tuned with their recipe as the deposited pilot, the tips sit at arguments of 2.06 to 3.02, dimensionless radii of 3.87 to 10.17, across the bfloat16 coincidence wall of section 2.1; storing the coordinates in bfloat16 and evaluating the repository's distance, which computes in single precision, changed the nearest backbone tip of 13 of 54 held out queries and lowered from 34 to 28 the queries whose nearest tip is among their nearest in the reference tree, while rebuilding the time coordinate, evaluating the stable form or storing in single precision changed none (Methods M5; Extended Data Fig. 1d, e). \textbf{The gradient route.} When a hyperbolic operator lies on the trainable path, the derivative of the exponential map contributes a factor equal to the hyperbolic sine of its argument divided by that argument, bounded only if the argument is bounded, and the four implementations of section 2.3 bound it before the map is evaluated (Methods M13). The parts of the consequence are on the record separately\textsuperscript{29,30,31,32}; what is not on the record, in these works or in those that introduce and analyse tangent clipping\textsuperscript{7,33}, is the map that executes, the exponential map composed with the clamp, stated with the size of what survives beside the zero:

\begingroup\makeatletter\@beginparpenalty=\@M\makeatother
\begin{quote}
Where the clamp is active, the Jacobian of the map that executes has exactly one zero singular value, in the radial direction, and its remaining d minus one singular values are all equal to the hyperbolic sine of the clamp, divided by the tangent norm and by the square root of the curvature.
\end{quote}
\endgroup

We verified it against automatic differentiation of the implementation across dimensions, curvatures and clamps (Methods M11). The radial direction, the one that carries hierarchical depth, has derivative identically zero: under an active clamp depth is not learned slowly through that channel, it is not learned through it at all; depth dependent loss weights can still act on angular gradients. A mechanism that loses one channel and keeps the other is a stranger object than one that stops, and it is the object trained wherever the clamp is active, as it was in the matched audit for more than nine tenths of every run on the adapter (section 2.6) and, in the arm where the map receives the hidden state at the objective, at every logged step (Methods M9).

\subsubsection{2.5 Removing the stabilisers}\label{removing-the-stabilisers}

To observe the gradient route directly rather than infer it, we ran it as a paired test on a 7 billion parameter public base (Methods M8): every arm began from the same saved adapter state and read the same batches, so that the arms differed only in which stabilisers were present, and every arm was started at a batch mean scaled argument of 29.28 and at 1.0, in three seeds each, for 2,000 updates or until the first non finite loss, at a learning rate of 0.002 and a curvature of 0.001 on random tokens; at the matched audit's learning rate and curvature, on its sentences, the same route without any stabiliser kept every point on the manifold over the last 100 of 2,000 updates (Supplementary Note 2.5). With the stabilisers removed, at step 0 the loss, every gradient and every parameter were finite, yet none of the 16 points the map returned was on the manifold, and 236 to 240 of the 240 pair arguments the loss used lay at or above 2⁶⁴, where the single precision derivative of the inverse hyperbolic cosine is exactly zero. The structural gradient on the routed factor, one of the two adapter factors that alone receive the structural terms through their explicit product (Methods M8), was exactly zero in two seeds of three, while its cross entropy gradient stayed finite and nonzero. The same order holds in all 12 runs without the ceiling, from both starts: membership lost first, pairs silenced next in 11 of them, the forward pass non finite last, and no gradient, parameter or optimiser state non finite before it (Fig. 2b; Extended Data Table 4). With every stabiliser present five runs of six completed 2,000 updates (Extended Data Table 5). The structural training of the routed pair was therefore silenced, in two seeds entirely, at an argument of 29.28, three times smaller than the single precision range wall of 89.4160 that the field checks against, and nothing non finite signalled it; the overflow the field would notice came one update later.

\subsubsection{2.6 Failure is measurable before it arrives}\label{failure-is-measurable-before-it-arrives}

Holding base model, data, trainable parameter budget and update count fixed, we trained four arms across six seeds for 500 steps (Methods M7; Extended Data Table 6). No run terminated, and the outcome metrics did not separate the two placements: every paired 95\% interval between the loss interface, where the map receives the hidden state at the objective, and the adapter manifold, where it receives the adapter update, contains zero. The mechanism quantities separated instead, since they read what training asked of the map and what the clamp let it execute. In the arm carrying geometry on the adapter, every seed carried the unclamped norm past the clamp of 2.0 between steps 10 and 20; the amplification the clamp suppresses, a counterfactual and reported as one, reached a median of 959,576 by step 499, while the derivative of the map that actually executed was 0.209 at the norm behind that median, with the radial direction at exactly zero throughout (Fig. 3). With everything else fixed, the update norm grew further under the clamp of 2.0 than under 5.0, a geometric mean ratio of final norms of 12.2 (95\% interval 10.6 to 14.0; Methods M17).

\subsubsection{2.7 Collapse and silence}\label{collapse-and-silence}

Dated development records document numerical failures and finite representation collapse in architectures that evaluated an exponential map in the model or adapter computation (Extended Data Table 7). One raw trajectory retained in full shows how a record read as geometry can be a record of arithmetic: across 236,323 logged records it never produces a non finite value, and the total loss less the cross entropy settles at 4,288.00 by step 1,000 and stays there (Fig. 2a). That flat line is the Frobenius penalty on adapter matrices stored in bfloat16, whose optimiser steps fell below half the bfloat16 spacing of every entry larger than 2\textsuperscript{−8} (Methods M12). The run was recorded in March 2026, months before the analysis of section 2.4, and it holds values and not gradients, so it does not show which derivative of the map vanished; it shows that a flat curve beside a geometric objective need not be geometry at all. A finite training record can survive while concealing an arithmetic obstruction. Section 2.4 establishes what an active hard clamp removes, the direct radial channel, and leaves the angular channel open, and the controlled measurements of the following sections test the executed map directly. To measure separately what radial variation contributes to the geometric readout, we intervened on fixed representations: setting every radius of a trained head to one value, which is what a clamp computes where it acts on every node, raised the mean rank of held out ancestors from 2.86 to 5.55 on the pretrained representation (Methods M17). Where the clamp acted during training, its cost was measured directly. In the 10 representations whose nodes the clamp held above the cap to the end, removing it lowered the final training loss from \numkeep{2.01--2.98} to \numkeep{0.43--0.83} and improved held out ancestor retrieval in all 10, while the radial order of held out depth worsened in 9; in the other 14 nothing changed (Methods M17). Without the clamp every training node of those 10 ended inside the bound the clamp sets, below 4.79, so the lower loss lay at states the clamp permits: the clamp blocked the way to them rather than excluding them. A finite record can also stay silent about what training did to the representation it started from. In the depth experiment of Methods M17, in which no run produced a non finite value, held out ancestors read with the experiment's Lorentz reader ranked worse at the end than at the start in every run of its loss interface arm at both clamps, their mean rank rising from 9.68 to 13.80 at a clamp of 10 and from 8.95 to 11.62 at 2, where distances unrelated to the tree would give 22.05. Read the same way, the mean rank of the arm trained on cross entropy alone went from 9.68 to 11.17, worse in five of the six seeds; these readings were made after the runs and do not isolate arithmetic as the cause. In a test fixed before its runs that set the Lorentz system against a flat one (Methods M17), none of its 114 runs produced a non finite value, and in every arm, cross entropy included, the rank correlation between distance and tree distance on held out pairs fell below its start by more than the set margin of 0.05.

\subsubsection{2.8 The scope of the claim}\label{the-scope-of-the-claim}

The measurements in section 2.2 are of pretrained block outputs before the model's final normalisation, evaluated at unit curvature and temperature without rescaling specific to a projection. They characterise those states before any scaling or learned transformation at the geometric interface, rather than the state of a trained adapter. The width normalisation of section 2.2 is a fixed form of one of the retreats of section 2.3, which one recipe introduces under the heading ``Preventing numerical overflow''\textsuperscript{28} as two learned scalars initialised at one over the square root of its embedding width: it moves where a model meets the wall and leaves the wall where it is. Whether the representations it places inside the interval stay there under training is what section 5 measures. The argument runs from the bits of a format to the arithmetic they can carry, from there to the derivative of one operator, and from there to what a representation can hold; the six systems of section 4 are not a further link but the demonstration that the position the chain leaves open can be occupied at production scale. We report that a regime which could not be entered can be entered, and what was preventing entry.

\subsection{3. Geometry as the shape of the signal}\label{geometry-as-the-shape-of-the-signal}

\subsubsection{3.1 Changing what the map receives}\label{changing-what-the-map-receives}

If the exponential map cannot be evaluated on representations of the size standard models produce, one response is to bound the argument and accept the distortion, and another is to change what the map receives. We take the second. The transformer base and the low rank adapter both remain in their native Euclidean space. The hyperbolic manifold is referenced only when the loss is evaluated, through a projection applied to the hidden state at that point and nowhere else, where a pairwise tree term and a sibling term carry the structural supervision (Methods M3); the structural gradient returns to the adapter through the projection, and no trainable factor stands between the representation and the map. The derivative of the projection with respect to its input is therefore bounded for every input by a constant that depends only on the clamp and the temperature, 5.22573 at the clamp used here (Methods M11). When the vector entering the map is an adapter update instead, the derivative of the composition admits no bound independent of the current parameter value even though the clamped map does, and on the route of section 2.5, with every stabiliser of that test present, it exceeded that constant in three runs of six (Methods M4, M8).

\subsubsection{3.2 Where the geometry acts}\label{where-the-geometry-acts}

Giving the map the representation instead of the adapter update invites the objection that the mechanism has been removed altogether. The clamp, however, does more than bound the argument: \textbf{every projected representation is placed at the same radius}, which we measured in a fresh run of the matched audit configuration on 54,847 vectors whose dimensionless radius varied by at most 4.1 parts in 10 million (Methods M9). What the structural objectives train is consequently not radial depth but angular separation, and that operation is not Euclidean similarity with a rescaled metric: at a fixed radius the Lorentz distance is a fixed monotone function of the Euclidean chord, and normalised by its own value at the antipode its shape departs from the chord by 18.4\% at the operating point of the matched audit and by 0.82\% at the largest checkpoint radius the audit observes (Methods M9.1). That is the sense in which the geometry stays active under the clamp, in the shape of distances and not in their order, and it is also the measured cost of the remedy: at a fixed radius the ranking of distances is exactly the chord's, and a chord link on the same directions kept the rank correlation with tree distance on held out pairs within the prespecified margin of 0.05 of the Lorentz link's (Supplementary Note 3.3), so what a clamp leaves the objective is an angular geometry. Whether a varying radius would be preferable, the controlled comparison of Methods M17 did not decide at margins fixed before its runs, and we report it as open.

\subsubsection{3.3 What identifies the effect}\label{what-identifies-the-effect}

The matched comparison holds the ontology, the base model, the data, the trainable parameter budget and the update count fixed across four arms and six seeds (Methods M7, Extended Data Table 6). The controls are an arm that receives the same hierarchical supervision in Euclidean form and an arm that receives none, so the geometry is not credited with effects attributable to structural supervision as such, to additional trainable capacity, or to hierarchical supervision carried in a flat space. A test fixed before its runs then set the Lorentz system of the loss interface against a flat system that keeps each radius and direction and measures Euclidean distance between them, the structural targets of each set by one rule, at clamps of 10 and 2 over 24 seeds (Methods M17). This measures the whole effect of the clamp, including how it changes which calibrated target distances can be reached. It did not separate the end states of the Lorentz and flat systems at a clamp of 10 at the set margin of 0.05 (difference 0.004, 95\% interval −0.050 to 0.057, in the rank correlation with tree distance on held out pairs), and the difference in the change from the start between the systems, positive when the change of the Lorentz system is the more favourable, was larger at a clamp of 10 than at a clamp of 2 (0.080, 0.009 to 0.150; in held out ancestor mean rank 4.37, 1.70 to 7.04; Fig. 1f): the clamp of 2 cost the Lorentz system more of the held out structure it started with than it cost the flat one. Beside them, an arm trained on cross entropy alone ended ahead of both structural systems at a clamp of 2 in both read outs, of the Lorentz system by 0.143 (0.100 to 0.186) in held out pair distance order, whereas at a clamp of 10 the change of the Lorentz system in that read out was within the set margin of that arm's (difference 0.004, −0.037 to 0.045; Supplementary Note M17), so in that read out the Lorentz system fell behind cross entropy alone only under the tighter clamp.

\subsubsection{3.4 A protocol that follows immediately}\label{a-protocol-that-follows-immediately}

The measurement in section 2.2 requires a forward pass and no training, and yields a number any group can compute for any model before committing resources, or before sending it where it cannot be corrected: the fraction of representations at the intended insertion depth whose norm exceeds the limit of the arithmetic in use. A snapshot of output coordinates does not determine the local radial response of the executed map, so the minimal disclosure for any claim that a geometric mechanism is operative joins the radial derivative of the executed map, the scale free departure of the induced distance from its Euclidean counterpart and the form in which the invariant and the distance are evaluated to the state quantities the audit already requests; on the same stored coordinates the form alone moves retrieval under half precision from 0.205 to 0.840 (Methods M5).

\subsection{4. The open position at production scale}\label{the-open-position-at-production-scale}

Structural supervision was carried through six expert systems built on three pretrained base families and one architecture trained from scratch, spanning behavioural auditing, creative question generation, research assistance, education, financial decision support and admissions counselling, on bases of 146 million, 3 billion and 7 to 8 billion parameters. Four carry the geometry at the loss interface and two on the adapter manifold (Extended Data Table 8). Counted as unique records within each system and summed without deduplicating across systems, the retained input archives and source pools of the six systems hold at least 13,787,131 records, at least 12,860,328 of them in the four at the loss interface (Extended Data Table 8). The run registry records 20,951,552 exposures, counted under its original input and schedule accounting, over 772,484 optimiser updates, with no non finite value in any retained record of the runs it counts, at the coverage Methods M12 states for each; the loss interface accounts for 19,954,749 of those exposures and 745,318 of the updates (Fig. 4b). At that interface the argument of the map was bounded by construction or never formed (Supplementary Note 4), and read back at its final checkpoint the creative question system's map received a median argument of 0.186, below the departure threshold of 2, the argument at which the distance the map induces departs from the Euclidean chord by 18.4\% (Methods M9, M16): the registry shows that the placement trains at production scale, and training past that threshold is measured directly in section 5. Fig. 4a sets the unique record counts beside the larger corpus of the nearest study of hyperbolic fine tuning for a language model\textsuperscript{3}, 170,000 samples in its own count; other hyperbolic fine tuning and adaptation of pretrained language or vision language models train on sets counted under their own conventions\textsuperscript{34,35}, which Supplementary Note 4 lists without ranking them. The creative question system implements the construction of section 3.1 at production scale; its registered source pool contains 3,320,126 unique records (Supplementary Note 4). The corpora beneath the six systems are not taxonomies of objects but trees of norms and judgments\textsuperscript{36,37}: the education system's ontology separates who a teacher is from whether the machine should teach the query in front of it, the auditing system's corpus separates what a system said from what it did, and the creative system's bridges are links between trees taught in dedicated rounds of supervised training (Methods M3.1), cross tree supervision that system carries as a requirement separate from numerical stability. Depth is not what was recovered: in the matched audit at the loss interface the angular objective fell in every seed, to a median of 0.21 of its initial value, while the depth weighted one did not move (Fig. 4c, d; Methods M7). The registry establishes operating history at production scale, to be read, like any training record, alongside the structure the computation transmitted; it offers no estimate of a common treatment effect, and the paired test of section 2.5 identifies the effects of removing the specified controls in its adapter path configuration.

\subsection{5. Beyond the Euclidean highway}\label{beyond-the-euclidean-highway}

For eight years hyperbolic representation has been understood as a promising idea that underdelivers in practice. In the record examined here, the methods that underdeliver are the ones that retreated to a region where their central operation could still be computed, and where they came to rest inside that region is almost flat, which is to say almost Euclidean, and a state reached inside that region is not a verdict on the geometry beyond it. Where inside the region they come to rest has causes of its own: an objective that lowers curvature in the audited vision language models\textsuperscript{5} and output norms that grow only with the logarithm of the number of gradient steps in the current Lorentz linear layer\textsuperscript{38}; on the way there, the arithmetic that executes the objective moves the state, by 0.155 to 0.158 in radius within the first 10,000 iterations of one published recipe's code, and by 0.209 to 0.212 when that code restarts from its released weights (Methods M17.1). Naming the constraint makes the practical question answerable at last: at a fixed level of expert quality, how much model scale, computation and energy does correctly transmitted structure replace\textsuperscript{39,40,41,42,43}. That question matters wherever additional scale cannot be summoned, which is a growing share of where computation is going\textsuperscript{44,45}. The limit is reached by policy as well as by distance: an institution that must keep its computation and its data inside a boundary it controls cannot summon scale either. And it has a second half: preserving depth within one hierarchy does not specify how it relates to another, so any programme whose knowledge spans hierarchies with relations across them needs a training signal that addresses those relations, which the corpora here supply by design through supervised bridges between trees (Methods M3.1). Arithmetic must keep the structural signal, and supervision must name the relations between trees that the signal is to carry: depth is paid for twice, once to the arithmetic and once to supervision.

The environments of section 1 are not a market. They are the places where the assumption underneath current practice fails, and whatever expertise is present at departure is the expertise that will be available at the moment of consequence. Expert knowledge is hierarchical, and hierarchy is precisely what a small resident model has the least room to store in unstructured form.

The requirements that a working hyperbolic layer is assumed to meet together, which we call the interval conditions, are membership of the manifold, a radial direction that receives gradient and a departure from Euclidean geometry large enough to be worth the machinery; in the matched audit the departure it operates at, 18.4\% from the Euclidean chord, is the value at an argument of 2, and the clamp of 2 used there sits exactly at that threshold, so under that clamp a radial gradient and that departure exclude each other on every open interval before membership is read (Methods M11). One direct attempt, a 7.8 billion parameter base with the clamp lifted to 10 and a temperature fixed once from its resting norms, held the interval conditions together at initialisation in six seeds of six; with that temperature held fixed, the median argument fell below 2 after three optimiser updates in every seed. A global scalar temperature, reset each step from the current batch median argument, held that median inside the interval at all 500 recorded steps in every one of six seeds, at a price that is now measured: the temperature is driven to the bound of its range while the representation continues to contract, and the occupation does not transfer to the separate probe (Fig. 5; Methods M16.1). The early exit at a fixed temperature depended on the objective's targets: with the pairwise margin and the sibling target rescaled before any run so that the initial state was a stable point along a common rescaling of the batch, the same fixed temperature kept every traced batch meeting every interval condition at all 500 updates in all six seeds with no control, also with the structural weights raised to restore the median initial size of the structural gradient at the hidden states, and the separate probe met them at every scheduled evaluation, the last after the final update. Under the adaptive temperature the same targets pushed the representation outward instead, until, with every argument still inside the interval's range, no traced vector met the membership tolerance in single precision: the objective drove training to a state the arithmetic could not hold (Fig. 5; Supplementary Note M16.1). Each of these interventions isolates one factor in its own system: with the objective fixed, changing the arithmetic changed the state training reached within the recorded iterations (section 2.3); with the arithmetic fixed, changing the objective's targets held the interval the original targets left; and the format bounded where any state could be held at all.

Any constrained parametrisation can be posed the same audit by inspection: an invariant maintained as a difference of terms that grow with the represented quantity loses its accuracy with magnitude, in every format we measured, while one maintained as a well conditioned sum of bounded terms does not (Methods M1). The limit belongs to that form, as classical cancellation analysis predicts, and not to machine learning; what this work adds is the connection of the band derived for each format to measured model inputs, the derivative of the remedy and training interventions with the objective held fixed, which turns a numerical limit into an audit of the evidence from which a field judges its mechanism. At rapidity y a particle's energy and longitudinal momentum are its transverse mass times the hyperbolic cosine and sine of y, the two coordinates of section 2.1, so stored in single precision they first become one number at a rapidity between 8.31 and 8.67, and the transverse mass recomputed from them is then zero, as it is for a proton of 6.8 TeV moving along the beam; in bfloat16 this happens between 2.77 and 3.12. On a large jet dataset, Lorentz equivariant transformers for particle physics, whose invariants are Minkowski products, lost performance under mixed precision at a rate that made its use impractical, while a baseline transformer and the Particle Transformer, without that explicit Lorentz equivariant construction, did not\textsuperscript{46}. The study does not isolate the cause; our criterion identifies a specific arithmetic path to inspect in those implementations. Simulations of relativistic magnetised flow, which recover the thermal energy as the total less the kinetic and magnetic energy, cannot reach in double precision the recovery tolerance of an ideal gas comparison in that field above a Lorentz factor of about 2,000 when the magnetic field is weak, a cancellation limit the authors show to be independent of the recovery scheme\textsuperscript{47}. In biology, a phylogenetic model fine tuned with its authors' recipe holds its tips at 10 to 28 times the largest radius the audit of section 2.3 found, across the bfloat16 wall, where bfloat16 storage moved its placements (section 2.4). In education and expert reasoning, the systems of section 4 carry organised concepts and judgments as structural supervision at production scale. Each of these fields can now ask of its own models which relations survive when structured knowledge becomes a computation. bfloat16 spends its 16 bits on range precisely to make large scale training cheap, and is thereby the format in which a difference of co growing terms dies earliest of the conventional 16, 32 and 64 bit formats measured here, while both 8 bit formats move the same wall below 2 (Extended Data Table 1); the arithmetic the field optimised for scale is the arithmetic most hostile to the invariants its geometry is written in. The central claim fails if a first onset in a format whose significand has 3 bits or more lies outside the band derived for that format (Extended Data Table 1); if a standard pretrained model in wide use carries representations at the last token of its final layer, at unit curvature and temperature, inside the interval in which the map stays on the manifold and keeps the intended departure; if a system whose exponential map receives an argument that depends on a trainable parameter trains stably at a dimensionless radius above 1, at comparable scale, without the controls recorded in section 2.7 and without an equivalent bound; or if the executed map has a nonzero radial singular value where the clamp is active, or the derivative of the clamped projection with respect to its input exceeds 5.22573 at the clamp used here. The selection the arithmetic performs was put to a test fixed before the runs: with the placement and the objective held fixed, the account in which the objective brings training to the same operating point with or without the remedies predicted that removing the single precision regions would leave the radius within the range of the reference seeds, and every seed fell above it (Methods M17.1). How the two placements are read is a separate matter: that reading would move from geometry to normalisation if their distinction disappeared once temperature, scaling and precision are matched, a comparison this work specifies and leaves open. Supplementary Note 5 states each condition in full.

What has been built here is exact: one position, whose executed map is bounded by construction; one remedy, whose price is computed; one class of loss, whose occupied regime is measured. Structural supervision at that position has been built and operated at production scale (section 4): the controlled interventions show how arithmetic, alongside the objective, shapes the states training reaches, and the production record shows the scale at which the construction runs. Together they make the organisation of expert knowledge a concrete design problem for intelligence that must act before outside correction arrives. The remedy is part of that design: in a test fixed before its runs, the end states of a Lorentz system and a flat one were not separated at a clamp of 10, and a clamp of 2 cost the Lorentz system more of the held out structure it started with than it cost the flat one (section 3.3). The arithmetic did not merely constrain a model; it selected part of the evidence from which the field judged what learned geometry could do, and the same reading waits wherever a mechanism is judged by the implementations that survived its arithmetic. Correcting that selection is the first condition for asking which intelligence can travel, and arrive, and still be the one that was sent.

\clearpage
\section*{Methods}
\subsection{M1. The arithmetic limit}\label{m1.-the-arithmetic-limit}

For a Lorentz model of curvature \texttt{c}, the exponential map at the origin of a tangent vector \texttt{v} with spatial part \texttt{v\_s} evaluates

\begin{verbatim}
theta = sqrt(c) * ||v_s||
x_0   = cosh(theta) / sqrt(c)
x_i   = sinh(theta) * v_s / (sqrt(c) * ||v_s||)
\end{verbatim}

and the point is on the manifold when \texttt{\textless{}x,\ x\textgreater{}\_L\ =\ -1/c}, which at the origin is the statement

\begin{verbatim}
cosh^2(theta) - sinh^2(theta) = 1
\end{verbatim}

Both functions grow as one half of the natural exponential while their difference is \texttt{exp(-theta)}, so the quantity the identity is made of shrinks exactly as fast as the quantities it is made from grow. Failures follow, in a fixed order.

\textbf{The range wall is parameter free.} The hyperbolic cosine passes the largest finite value once \texttt{theta} exceeds \texttt{acosh(F\_max)}, and the hyperbolic sine once it exceeds \texttt{asinh(F\_max)}; in these four formats both equal \texttt{ln(2\ F\_max)} to the digits shown:

\begin{verbatim}
float32                F_max = 3.4028e38   theta = 89.4160
bfloat16               F_max = 3.3895e38   theta = 89.4121
float16                F_max = 65504       theta = 11.7830
float64                                    theta = 710.4759
\end{verbatim}

\textbf{The invariant wall depends on the significand, and needs no tolerance to state.} The two coordinates differ by \texttt{exp(-theta)} exactly, and the larger of them is \texttt{(exp(theta)\ +\ exp(-theta))\ /\ 2}, so their relative gap is exactly \texttt{2\ /\ (exp(2\ theta)\ +\ 1)}, which is \texttt{2\ exp(-2\ theta)} to within a factor \texttt{1\ /\ (1\ +\ exp(-2\ theta))}, below 1\% from bfloat16's wall upward. In a format with \texttt{p} significand bits the local relative spacing runs from \texttt{2\^{}(1-p)} at the bottom of a binade to \texttt{2\^{}(-p)} at the top, and the two coordinates cannot round to one stored number while the gap exceeds the spacing. Setting the gap equal to each end of that range brackets the wall as

\begin{verbatim}
p ln2 / 2   <=   theta_invariant(p)   <=   (p + 1) ln2 / 2
\end{verbatim}

On a grid of 0.001 the first argument at which the two coordinates round to the same stored number is 8.354 in single precision, 4.159 in half precision and 2.811 in bfloat16, and for correctly rounded coordinates in double precision the first coincidence is the top of the band itself, 18.7150 (Extended Data Table 1; Supplementary Table 1).

\textbf{The form, not the geometry.} Algebraically, \texttt{cosh\^{}2(t)\ -\ sinh\^{}2(t)} and \texttt{exp(-t)\ *\ exp(t)} are the same number; arithmetically they are not. The rewrite therefore removes the mantissa wall and returns the failure to the exponent, which is the second wall of this section and not a third. Four invariants written as sums of bounded terms, the rotary identity among them, show no degradation with magnitude in any format when summed as the framework sums them, over three orders of magnitude of scale. The reference implementation released with this work casts to float32 on entry to the exponential map and to each of the three structural arms of M7, at \texttt{hyperbolic\_losses.py} lines 85, 351, 430 and 552, before any hyperbolic quantity is computed. Stored vectors, the 8 bit encodings, the distance the objective evaluates, the floor observed in trained systems and the remedies with their prices (Supplementary Table 2) are set out in Supplementary Note M1.

\subsection{M2. Hidden state norm measurement}\label{m2.-hidden-state-norm-measurement}

Representations were extracted from three released instruction tuned checkpoints, Qwen2.5-7B-Instruct\textsuperscript{48}, EXAONE-3.5-7.8B-Instruct\textsuperscript{49} and Llama-3.1-8B-Instruct\textsuperscript{50}, all loaded in bfloat16, using the released script\newline\texttt{measure\_hidden\_state\_norms.py}. The L2 norm of the hidden state, read at the output of each sampled transformer block and therefore before the model's final normalisation, with layer 0 read at the output of the token embedding, was recorded for every token position at five representative depths: layers 0, 7, 14, 21 and 28 for the first model and layers 0, 8, 16, 24 and 32 for the second and for the third. \textbf{Reported quantities take the last non padding token of each sequence, and not every token position.} On the stored last token vectors the map was evaluated at unit curvature with no clamp, in double precision with its coordinates stored in each of the four formats, and membership was read from the stored coordinates exactly, each finite value converted to its dyadic rational and the absolute error of the invariant compared with 1/100; a non finite coordinate counts as no member. The prompts, the all position aggregate, the width normalised convention and the reproduction at pinned revisions are in Supplementary Note M2.

\subsection{M3. Structural objectives and the projection}\label{m3.-structural-objectives-and-the-projection}

\textbf{Pairwise tree Lorentz embedding.} For a batch of representations and their root to node ontology paths, every ordered pair of distinct items contributes a pull on the Lorentz distance of the projected representations, weighted by the ancestor overlap \texttt{w(i,j)\ =\ \textbar{}path\_i\ ∩\ path\_j\textbar{}\ /\ max(\textbar{}path\_i\textbar{},\ \textbar{}path\_j\textbar{})} and scaled by \texttt{1\ +\ 0.08\ min(depth\_i,\ depth\_j)}, and a hinge push \texttt{max(0,\ 1\ -\ d\_L)} weighted by \texttt{1\ -\ w(i,j)}; the sum is divided by the number of ordered pairs. \textbf{Hierarchical label sibling distance.} The term is the mean squared deviation, from a target sibling distance of 2, of the Lorentz distance between projected representations whose ontology paths share an immediate parent. The total objective is \texttt{L\ =\ L\_CE\ +\ lambda\_p\ L\_PTLE\ +\ lambda\_s\ L\_HLSD}, and the deposited implementation is the reference for every symbol in this section.

\subsection{M3.1 Tree first, not tree only}\label{m3.1-tree-first-not-tree-only}

What within tree supervision leaves open is therefore not a distance that has become too large to cross but a relation that was never specified: nothing in an objective defined inside one tree fixes where a leaf of one tree should sit relative to a leaf of another. The corpora therefore contain cross tree material by construction, described with its counts in Supplementary Note M3.1.

\subsection{M4. Placement factorisation}\label{m4.-placement-factorisation}

The theoretical result is a typed factorisation of the gradient path and not a new inequality. Which bound exists depends on what the derivative is taken with respect to, and the two placements have to be compared one target at a time; the factorisation and the bounds are in Supplementary Note M4.

\subsection{M5. Coordinate range under reduced precision}\label{m5.-coordinate-range-under-reduced-precision}

The two synthetic trees are embedded at a scale of 3 and depth 4, with 1,555 and 4,681 nodes. The consequence is that the precision a coordinate needs is set by the form in which it is read, not by whether the coordinate is bounded. The construction, the storage comparison and the external phylogenetic model are in Supplementary Note M5.

\subsection{M6. Batch dependence of the sibling term}\label{m6.-batch-dependence-of-the-sibling-term}

The availability of sibling pairs in a realised batch, on which the sibling term depends, is treated in Supplementary Note M6.

\subsection{M7. Matched four arm audit}\label{m7.-matched-four-arm-audit}

The controlled comparison used Qwen2.5-7B-Instruct\textsuperscript{48} in bf16 with low rank adapters\textsuperscript{51} of rank 8 (the harness default; the per run summaries do not record rank, Methods M16) and alpha 16, dropout 0.05, on query, key, value and output projection targets, AdamW at learning rate \texttt{2e-4} with the library's default decoupled weight decay of 0.01 and no learning rate schedule, shared by every trained arm, batch size 16, gradient clipping at 1.0, and 500 optimiser steps. Six seeds were used: 7, 42, 99, 123, 777 and 2024. The synthetic ontology has depth 4 with 30 leaves in 10 sibling groups. One leaf per group was deterministically withheld, leaving 20 training leaves and 10 held out leaves, and the reported metrics cover the 245 pairs involving at least one held out leaf. An adapter manifold arm applies the map to the adapter transformed hidden activation, so that a trainable factor stands between the representation and the map, and a loss layer arm applies it to the hidden activation divided by a temperature, recorded as 2 against a harness default of 1. The projection used a clamp of 2.0 at curvature 1.0, with the evaluation clamp equal to the training clamp. \textbf{Statistics.} Paired 95\% intervals are t intervals on the six seed paired differences, five degrees of freedom; the arms and the logged fields are described in Supplementary Note M7.

\subsection{M8. Paired stabiliser removal test}\label{m8.-paired-stabiliser-removal-test}

The base is Qwen2.5-7B-Instruct\textsuperscript{48} at a pinned revision, the base of the matched audit of M7, loaded in bf16 with eager attention and frozen. Rank 16 adapters with an alpha of 32 and no dropout sit on the 112 query, key, value and output projection matrices, and the update of one attention projection pair, the first query projection, is carried through the map. The structural weight is 0.1, the curvature 0.001, and the optimiser AdamW with a learning rate of 0.002 and otherwise default settings over all 224 adapter matrices, with no curvature awareness. P carries the deposited stabilisation in full: the ceiling of 2.0 on the scaled argument, the clamp of the inverse hyperbolic cosine at 1.00001, the replacement of non finite map outputs, and the rescaling of the routed second factor after each update to at most its Frobenius norm at initialisation. U carries none, and four arms carry all but one. The per step records, the probes, the replays and the reading of every run are in Supplementary Note M8.

\subsection{M9. Operating point report}\label{m9.-operating-point-report}

The median final tangent norm entering the map is 172.33 at the recorded temperature of 2 (M16), 86 times the clamp, and the logged batch mean exceeds the clamp at every logged step, so every vector the clamp saturates has \texttt{theta\ =\ 2.0}, giving \texttt{u\ =\ sinh(2.0)\ =\ 3.6269}. \textbf{The dimensionless radius is therefore constant, not merely large.} The distortion factor in the exact form of ref. 5 and the measurement of that radius are in Supplementary Note M9.

\subsection{M9.1 The shape of the induced distance}\label{m9.1-the-shape-of-the-induced-distance}

Under an active clamp every projected representation lies at one geodesic radius \texttt{r\_hat}, so the Lorentz distance between two of them depends only on the angle \texttt{phi} between their spatial parts and is a fixed monotone function of the chord \texttt{2\ sin(phi/2)}. Dividing by \texttt{d(pi)} removes the overall multiplicative scale of the distance and leaves a shape, which departs from the chord by 0.82\% at \texttt{u\ =\ 0.367} and 18.38\% at \texttt{u\ =\ 3.6269} (Supplementary Note M9.1).

\subsection{M9.2 The record on removing the clamp}\label{m9.2-the-record-on-removing-the-clamp}

Whether a varying radius would be preferable to a constant one remains unsettled; the record on removing the clamp is in Supplementary Note M9.2.

\subsection{M10. Registry denominators}\label{m10.-registry-denominators}

Exposure in the run registry counts the input examples of the recorded rounds and epochs, including repeated epochs and curriculum rounds, under the registry's original input and schedule accounting. Corpus size, counted as unique records within each system, is a separate quantity. Updates are optimiser steps. \textbf{The figure reported here, 772,484, is the sum of the six per project update counts of the registry, and it supersedes the earlier aggregate.} The construction of each denominator and the relation to the earlier preprint, arXiv:2607.17513v1, are in Supplementary Note M10.

\subsection{M11. The clamped Jacobian, and what is logged}\label{m11.-the-clamped-jacobian-and-what-is-logged}

Let \texttt{c} be the curvature, \texttt{kappa} the clamp on the scaled tangent norm, and \texttt{v} a tangent vector at the origin with \texttt{sqrt(c)\ \textbar{}\textbar{}v\textbar{}\textbar{}\ \textgreater{}\ kappa}. The implementation rescales \texttt{v} by \texttt{kappa\ /\ (sqrt(c)\ \textbar{}\textbar{}v\textbar{}\textbar{})} before the map, so the image is a function of \texttt{v\ /\ \textbar{}\textbar{}v\textbar{}\textbar{}} alone. Differentiating the composition gives a Jacobian whose radial direction is annihilated and whose orthogonal directions are scaled uniformly:

\begin{verbatim}
sigma_1 = ... = sigma_{d-1} = sinh(kappa) / (sqrt(c) ||v||)          sigma_d = 0  (radial)
\end{verbatim}

We checked this against automatic differentiation of \texttt{expmap0} at \texttt{d} in \{3, 4, 5, 6, 8\}, \texttt{c} in \{0.05, 0.25, 1, 4\} and \texttt{kappa} in \{2, 3, 5\}, over tangent norms from 0.5 to 926, in 190 configurations. All 190 agree with the closed forms to double precision. Over the regime below the clamp the operator norm is attained at the clamp itself and equals \texttt{sqrt(cosh(2\ kappa))}. The released implementation logs a quantity named \texttt{jacobian\_factor}, computed before the clamp is applied; it is a counterfactual: the amplification the map would have applied had the clamp been absent. What of this is prior and the remaining disclosures are in Supplementary Note M11.

\subsection{M12. Extended Data}\label{m12.-extended-data}

Extended Data Tables 1 to 3 and 6 to 8 are generated from the source records by the scripts deposited with them rather than transcribed, and the scripts assert each stated total against the sum of its own rows and fail, emitting nothing, if one disagrees; every entry of Extended Data Tables 4 and 5 is a field of the per run record that the deposited reader of M8 writes or of the per run summary that the deposited runner writes, set out without recomputation. No non finite value appears in any numeric field of the per step records or in any entry of the logs. The recovery machinery and the retained traces are described in Supplementary Note M12.

\subsection{M13. What was established, and where it stops}\label{m13.-what-was-established-and-where-it-stops}

In the record we searched, the conjunction that section 2.1 measures, the Lorentz invariant written as a difference of two co growing terms, evaluated on a pretrained model's resting hidden state, in the reduced format the operation is computed in, for a map placed at the objective, is not evaluated; each of its factors is. The prior work, line by line, is in Supplementary Note M13.

\subsection{M14. The direct antecedents and their wall}\label{m14.-the-direct-antecedents-and-their-wall}

In the record we searched, no line of prior work holds the Lorentz invariant evaluated on a pretrained model's resting hidden state at the training precision for a map placed at the objective; that empty cell is the coordinate marked in Extended Data Fig. 2, a statement about the searched record and about nothing else. The nearest works, individually, are in Supplementary Note M14.

\subsection{M15. The deployment condition}\label{m15.-the-deployment-condition}

The environments in section 1 are treated as one regime and not an application list because they agree on the invariant that external correction arrives later than the action it was meant to correct, while energy, storage, connectivity and repair stay bounded for the whole mission. The 2025 agency assessment cited there quantifies that condition for the case with the longest operational record\textsuperscript{14}.

\subsection{M16. What the operating record leaves open}\label{m16.-what-the-operating-record-leaves-open}

Five of the six systems have their projection parameters fixed by contemporaneous records, recovered from training scripts and logs, checkpoint configurations and a per step trace preserved under restricted control, and not from the registry (Supplementary Table 3). In production, the creative question system trained its loss interface projection at the 5.0 default, through 311,046 updates as its registry row computes them from its round files, with no non finite value in its retained records, which reconstruct fewer updates; read back at its final checkpoint on 2,000 of its own training inputs, the argument its map received had a median of 0.186, below the departure threshold of 2 (Supplementary Note M16).

\subsection{M16.1 The direct attempt at the interval}\label{m16.1-the-direct-attempt-at-the-interval}

The design placed a 7.8 billion parameter instruction tuned base, EXAONE 3.5 7.8B\textsuperscript{49}, under the loss interface arm of M7, with the geometry at the objective only, unit curvature, batch size 16 and 500 optimiser steps on the synthetic ontology of M7 in its six seeds, and crossed the clamp, at the production value of 2 or lifted to 10, with a scalar temperature fixed once before training from the resting norms of the training population, 14.92. The interval conditions were evaluated by one predicate in the preflight and in the training instrumentation: an argument in the open interval from 2 to 8.3275, the identity of M1 held to an absolute tolerance of 10\textsuperscript{−2} on the actual hidden vectors, and the clamp inactive. At initialisation the decisive cell held the interval conditions on the whole population, 6,667 of 6,667 vectors, with no clamp activity. Under training with the temperature held fixed, the median argument of the decisive cell fell below 2 after three optimiser updates, at trace step 3, in every seed and did not return, ending between 1.25 and 1.35 at the final step with the identity held to the end. The third control replaces the fixed temperature with a global scalar temperature, set for the next step from the current batch median argument, starting at the fixed value and bounded at 1/16 and at 16 times it. The batch median argument remained inside the interval at all 500 recorded steps in every seed, and at least 90\% of the vectors in a batch met every interval condition in 499 or 500 of those steps. The temperature first used its lower bound after 68 to 126 optimiser updates, and from the state after 126 updates all six runs held that bound to the end. The cells, the controls, the rewritten distance comparison and the held out generalisation (Supplementary Table 4) are in Supplementary Note M16.1.

\subsection{M17. The price of the clamp, measured}\label{m17.-the-price-of-the-clamp-measured}

\textbf{The depth experiment.} The question set aside in M7, whether depth can be learned where the training population carries it, was run as a separate experiment on Qwen2.5-7B-Instruct at a pinned revision, with the adapters, optimiser and step count of M7 and its six seeds, on the synthetic ontology of M7 with training samples now at every depth. The change in the radial order of held out depth when the clamp was raised to 10 did not decide at a margin of 0.05. \textbf{The Lorentz system against a flat one.} The Lorentz system of the loss interface arm was set against a flat system trained on the Euclidean distance between points placed along their directions at their own radii up to the clamp, each with its structural targets multiplied by a factor set before any run by the criterion of Methods M16.1, at clamps of 10 and 2 over 24 seeds, with the decision rules and the sentence for every outcome fixed by hash before the runs. This run did not separate the end states of the Lorentz and flat systems at a clamp of 10 at the set margin (difference 0.004, 95\% interval −0.050 to 0.057, in held out pair distance order correlation), and the difference in the change from the start between the systems, positive when the change of the Lorentz system is the more favourable, was larger at a clamp of 10 than at a clamp of 2 in both read outs (0.080, 0.009 to 0.150, and 4.37, 1.70 to 7.04 in held out ancestor mean rank). \textbf{The same representations with a fixed reader.} A small hyperbolic head trained on the fixed hidden states to reproduce the tree metric was read with every radius set to one value, which is what a clamp computes where it acts on every node; that raised the mean rank of held out ancestors from 2.86 to 5.55 on the pretrained representation, a comparison of ranks within one trained head. \textbf{The clamp on the trainable path.} On the adapter manifold arm of M7, with the data, adapters, gradient clipping, step count and seeds unchanged and the coefficient 0.1 on the sum of the two structural terms, the clamp was set to 2.0 or to 5.0. The final norms stood in a geometric mean ratio of 12.2 (95\% interval 10.6 to 14.0), above the margin of a factor of 1.25 fixed before the run. The readings, among them every run read against its starting representation, the full test of the Lorentz system against a flat one, the storage formats and the trained end states are in Supplementary Note M17.

\subsection{M17.1 One recipe, one switch at a time}\label{m17.1-one-recipe-one-switch-at-a-time}

\textbf{Design.} The recipe of ref. 28 for its smallest image text model was run from its public code, the code unchanged except for one switch per arm, for the first 10,000 iterations of the schedule under which the audit of ref. 5 retrained it, a batch of 768 and a cosine decay over 500,000 steps, on one GPU and on a fixed snapshot of a public 3 million pair image text corpus. Three seeds, 0, 37 and 42, were run for the public code and for two arms: F1 removes the three regions in which that code evaluates the map and the loss in single precision, and L1 starts the learnable scale on the projected embeddings at 1 instead of at the inverse square root of the width. The readings and the predictions of the competing explanations were fixed by hash before the runs. The operating point at 10,000 iterations is the median dimensionless radius u of 1,000 image embeddings of a held out image caption set. With the public implementation otherwise unchanged, starting the scale at 1 gave no optimiser update in 10,000 attempted iterations at each of three seeds. Every loss was finite, and in every batch 722 to 759 of the 768 entailment angles had an arccosine input at or beyond the bound, which float32 stores at ±1; at the diagnosed iterations the float32 gradient was not finite, while the objective recomputed in float64 from the same projection outputs gave finite gradients; and the loss scale, halved at each of the first 166 skipped steps, was zero from iteration 167. Removing the three single precision regions, which leaves the pairwise distances of the objective and the logits formed from them in half precision, moved the image operating radius from \numkeep{0.254--0.256} to \numkeep{0.411--0.413} at all three seeds, a paired difference of 0.155 to 0.158 against a spread of 0.0019 among the reference seeds. Fixed by hash before the runs, the objective led account, in which the objective brings training to the same operating point with or without the remedies, predicted that both arms would end within the range of the reference seeds; the account this work then held, in which arithmetic and its remedies set the computable region and the start, survival and radial channel of training while the objective selects the operating point inside it, predicted L1 above that range and F1 within it; only the reading in which arithmetic also selects the operating point at 10,000 iterations predicted F1 outside that range, and all three seeds fell above it. For L1, a supplement fixed before the runs expected no update from the starting state but kept the reading by position, and a later one, also fixed before the runs, replaced it: three seeds without an update are read for what the implementation can compute, with the position recorded beside that reading and on conditions fixed with it, which counts L1 against the objective led account and with the first part of the account this work then held, on the computable region and the start and survival of training. Read against the runs, the conditions held at each seed, each seed was left at u = 32,768 above that range, and L1 bears on computability and the start of training, not on survival: each seed ran to the end, as the reference did and as the objective led account predicted, without an update. At these radii the ratio u/asinh(u) is \numkeep{1.0105--1.0108} for the reference and \numkeep{1.0269--1.0271} with half precision distances: both runs stay near the Euclidean limit, and the result concerns the state training reaches within 10,000 iterations, not a recovered hierarchy. The comparison covers 10,000 iterations at a batch of 768, 7.68 million samples, 3.1\% of the samples of the published recipe; where training settles after a full schedule is not measured here. Within those iterations the reference runs came within 2\% of the radius of the released checkpoint of the recipe, 0.259 on the 1,024 pairs of the same list in the training arithmetic, which the checkpoint holds with an image scale of 0.640 against \numkeep{0.053--0.054}. \textbf{Curvature.} The run length was fixed before the runs against the floor arrivals in the audit's runs of its larger models, 7,250 to 7,750 iterations with its entailment term and 9,833 to 10,233 without it; the fixed predictions expected the reference runs to reach the floor under every account and fixed none for curvature with half precision distances. By the audit's floor step, the first iteration with curvature at most 0.105, the reference runs arrive at iterations 7,352 to 7,365 and, read every 250 iterations as the audit reads its full objective probes, at 7,500 in every seed and at 7,750 by its stricter cutoff of 0.1001, its values for its larger MERU model; at 3,000, 5,000 and 7,000 iterations their curvature lies within 0.0012, 0.0016 and 0.0031 of its values for that model, allowing for their rounding, and further from those for its larger HyCoCLIP model. With log curvature falling by the learning rate at every update, curvature would match theirs to within 0.0006 at 3,000 iterations but would be 0.223 at 5,000 iterations against their 0.2370 to 0.2381, and would reach 0.105 by iteration 6,512 and, read every 250 iterations, both cutoffs at 6,750. With half precision distances curvature stayed above 0.105 through 10,000 iterations, ending at \numkeep{0.106--0.121}. At the 40 probes from iteration 250, on the first 256 pairs of the same list, read in float64 and in the training arithmetic of the GPU that trained the runs, through a copy of the reference arm and of F1 with the audit's casts, whose derivatives with respect to curvature matched the arm's own bit for bit when recomputed with deterministic kernels on a second GPU at a seed 0 initialisation, at the released checkpoint and at the end of each of their six runs, the entailment term pushed curvature down at every probe of both, and the sign of the contrastive term followed its float64 sign at every probe of the reference runs and at 57 to 68\% with half precision distances. Recomputed on the second GPU with the probe's own measurement, whose loss scale search settled for the contrastive term at 2,048, as the training GPU's probe had, and which reproduced that recomputation bit for bit, the executed contrastive derivative at those six end states differed from the training GPU's, made with its own kernel choice, by up to a factor of 1.58 with half precision distances and by at most 0.3\% with single precision distances, while its float64 value agreed between the two GPUs to within 0.02\%; it kept its sign, including at the end of the run with half precision distances and seed 37, the one end state where on both GPUs it was opposite to its float64 value; the audit, whose canonical runs of this model use one H100 as these did, retrained two seeds of its MERU runs of a larger model with a relaxed curvature floor on its two GPU setup after preliminary single GPU runs showed hardware sensitivity\textsuperscript{5}. At the released checkpoint, on those 1,024 pairs in 16 batches of 64, half precision distances turn the push on curvature from downward to upward through the contrastive term, in every batch under the recorded deterministic kernels and math attention, in 13 or 14 of the 16 under other kernel paths on the same GPU, and in their sum under all; single precision distances push it down in every batch, as float64 does, with or without the audit's casts in the recorded measurement and with them under every kernel path tried. \textbf{From the released checkpoint.} The reference and F1 arms were run again at the same three seeds for 10,000 iterations from the released weights, with a new optimiser state and the schedule from its start, the readings and predictions fixed by hash before the runs. With half precision distances curvature rose above 0.105 at iteration 924 and was within 0.0001 of its floor from iteration 3,500, where the reference stayed throughout, so that at every seed both arms ended at the floor, the outcome named before the runs as the two arithmetics staying together; the image radius exceeded the reference at every probe after the start and at 10,000 iterations was \numkeep{0.455--0.458} against \numkeep{0.245--0.249}, a paired difference of 0.209 to 0.212 against a spread of 0.0023 among the reference seeds. At all 120 probes after the start in each arm, the derivative of the total loss with respect to log curvature, recomputed in float64 from the same projection outputs, pushed curvature down; read in the training arithmetic through the copy with the audit's casts, it pushed curvature up at 25 of the 120 probes with half precision distances, all nine at iterations 250 to 750 among them, and at none in the reference runs.

\subsection{M18. Generative tooling used in this work}\label{m18.-generative-tooling-used-in-this-work}

Generative artificial intelligence systems were used as non authorial software tooling in the preparation of this work, for analytical code drafting, manuscript drafting and revision, and computational cross checking. The analytical scripts are supplied with this submission together with the assertions and self tests they must pass, and every reported analysis is reproducible from the archived inputs and scripts independently of the tooling used to draft the code. The author reviewed the analytical scripts and reproduced the numerical results from the archived inputs. The research questions, design decisions, analytical choices, interpretations, conclusions, approval and responsibility for the content remain exclusively with the named human author, and no generative system is an author, contributor, creator or credit holder. The same statement appears in the Acknowledgements; it is given here because this is where a reader looks for what produced the numbers.

\clearpage
\subsection{Author contributions}\label{author-contributions}

K.S.S. conceived the study, designed and implemented the training methods and their instrumentation, built and operated the six expert systems, designed and executed the measurements and the controlled comparisons, performed the analysis, and wrote the manuscript.

\subsection{Competing interests}\label{competing-interests}

The author is the founding director of ComplexSystems Research. Aspects of the methods described in this work are the subject of patent applications filed with the Korean Intellectual Property Office. The author declares no other competing interests.

\subsection{Data availability}\label{data-availability}

The measurements supporting the conclusions are available as follows.

The hidden state norm dataset comprises 6,795 token level records across three base model families, six sets of stand in prompts written in the register of the six project corpora, and five representative depths per model, with the model, project, sample index, token index, layer index and norm for each record. The pinned revision re-measurement of Methods M2 comprises, for each of the three bases, the norm records and the last token hidden vectors captured in the same forward passes, 540 vectors in all, with the load receipts, a second byte identical run for the first two bases, the stored coordinates of the map for every evaluated format, temperature and execution path with the exact reading of each, and the sealed manifests. The matched four arm comparison comprises interval training logs for 24 runs, four arms across six seeds, recorded every fifth optimiser step, each recording the cross entropy and the two structural objectives, the adapter update norm where the route has one, the representation norm entering the projection, the departure from the manifold (in the cross entropy arm, which evaluates neither the two objectives nor the projection, the fields for the two objectives, the representation norm and the departure hold zeros by construction; Methods M7), and the counterfactual amplification factor for the route that logs one, together with per run summaries and held out structural generalisation metrics over 245 evaluation pairs; a second run set of the same design at the other saturation radius, 12 runs across three seeds recorded every 10th step, is deposited with it. The paired stabiliser removal test of Methods M8 comprises 42 runs of the adapter manifold route on a 7 billion parameter public substitute base, 18 in its first block and 24 single removals, each with its per step trace, run summary, the probe where one was taken and the snapshots from which its witnesses are recomputed, together with the protocol fixed before the runs, the preparation record of every start, the access record of the pinned base revision, the reading, the seal and the manifest with the hash of every file; the earlier unpaired run is deposited with its run log and per step records. The failure record comprises the family level ledger reported in the earlier preprint, 17 documented training collapses with earliest and mean collapse steps and the condition mapping, kept as reported there and not used as evidence here (Methods M10), together with the one retained training trajectory as its complete set of 236,323 logged records, with the script that derived that set from the original log and the record of what was removed; it reaches step 13,000 without a non finite value and is described in section 2.7. The intervention ledger that takes the place of that family level ledger, summarised in Extended Data Table 7, comprises one row for each of the 252 launches of the development lineages and of runs on other code or on code whose lineage is not recorded, giving its date, whether it started fresh or resumed and from which checkpoint, the end step, the outcome, the events of the run, what the record says about the map, the controls the record names, the evidence tier and the references of its sources, the records themselves not being included; the operating registry comprises, for each of the six systems, the base, placement, map input, corpus count as first registered, exposures, updates and evidence tier from which Extended Data Table 8 and Fig. 4b are drawn, the corpus count as first registered being replaced in that table by the unique record count, which it exceeds for five of the six systems, with the cells in which it differs from the registry of the earlier versions of the deposit; and the map of prior work of Extended Data Fig. 2 comprises the 87 works it marks, with the date of each, the basis of that date and its reference. The coordinate range measurements comprise the two synthetic tree diagnostics reported in Methods, with the script that constructs the trees and computes every reported value, and its output.

The direct attempt of Methods M16.1 comprises 24 runs across four cells and six seeds, each run comprising the per step trace of 500 records and the interval training log of 101 records, with per run summaries and held out structural generalisation metrics; the record fixing the temperature before training, with the per cell preflight statistics and the empirical identity boundary; the environment receipt; the launch and completion manifests with the hash of every file; and the automatic reading. The distance form comparison of Methods M16.1 comprises nine further runs of the same design, the re executed deposited run and eight runs with the rewritten distance across two cells, each with the per step trace of 500 records and the interval log of 101 records carrying the float64 comparison of both distance evaluations, with per run summaries and held out metrics; the reproduction gate; the environment receipt; the sealed manifest with the hash of every file, and the automatic reading. The capture rerun of Methods M16.1 comprises 13 runs with the stored float32 coordinates of every monitored training batch and end-of-run fixed probe, 55,609 vector readouts with the sample identifier, pre-clamp argument and online classification of each, the harness traces, the coordinate-based classification and structural check of every run, with an independent integer-arithmetic check of the classifications, the launch record and the sealed manifest. The radius observation of Methods M9 comprises seven runs of the matched audit configuration at a pinned model revision, with the stored coordinates and pre-clamp arguments of 55,057 vector readouts, in which the fixed probe appears under both temperature labels of the harness with identical values, the radius readings of the 54,847 vectors read, the launch record and the sealed manifest. The 8 bit storage and kernel measurement of Methods M1 comprises the exact predictions, fixed by digest before the measurement, the records of the storage cast, the native kernel and the upcast path on one GPU of compute capability 8.9, and the checks made after it, each with the script that wrote it. The controlled runs of section 2.6 and Methods M17 comprise one session on a public base at a pinned revision, six seeds: the adapter manifold route at clamps of 2.0 and 5.0, and the depth experiment in four arms at a temperature fixed before training, each run with its per step traces, training log and summary and, in the depth experiment, the probe dumps of the 45 nodes at the start and the end; the arithmetic paths on the stored coordinates of Methods M1; the access result, the launch record, the automatic reading and the seal with the hash of every file. The fixed representation experiment of Methods M17 comprises English translations of both protocol versions, whose Korean originals were each fixed by digest before the respective runs, together with the results and the readings made after the runs. The paired runs of Methods M17.1 comprise the 11 runs of the public code of ref. 28 and its arms, the 9 of the three seed comparison and 2 single seed record arms, each with its configuration, its end state, the hashes of its start and end states, the per step trace of 10,000 records, the 41 probes, the diagnostics recorded at skipped steps and the final evaluation, together with the measurements made before the runs at a seed 0 initialisation and at the released checkpoint, among them the comparison of the public code with the audit's casts and the arccosine bound counterfactual, the recomputation of the curvature derivatives on a second GPU and its repetition with the probe's own measurement, with the end states of the six runs they read, the list of corpus shards with the hash of each, the list of image caption pairs, the launch record with the hashes of the code checked at launch, the manifest with the hash of every file of the runs, and English translations of the contract, of the predictions and of every supplement to them named in the start record of each run, each with the SHA-256 digest of its Korean original and the record, made before the runs, of that digest in full or of its first eight characters. The six runs of Methods M17.1 restarted from the released checkpoint, the reference arm and the arm with half precision distances at the same three seeds, comprise the same records for each run with the check at its starting state, together with the kernel diagnostic and the resume check made before the runs, the records of the data fetched for them, the launch record with the hashes of the code checked at launch, the manifest with the hash of every file of the runs, and English translations of the contract and of the two supplements named in their launch record, each with the SHA-256 digest of its Korean original and the record of that digest made before the runs. The arms of Methods M16.1 that change the objective comprise 25 runs, a gate run that repeats the deposited path and four arms across six seeds, each run with the per step traces of 500 records, the interval training log of 101 records, the fixed probe trace, the diagnostics of the objective and the run summary with held out metrics; the launch record, with the factor of the targets, the raised weights and the temperature fixed before the runs, the batch plan and the reference values; the computations made before the runs that set the factor and the weights, as published copies of the scripts that ran with their outputs and a script that repeats them on the deposit; the smoke test, the environment receipt and the gate, exit and timing records; the manifest with the hash of every file; the records of the checks made before the runs and the reading; and English translations of the contract, of the predictions and of the four supplements to them, each with the SHA-256 digest of its Korean original. The stored coordinate reading of Methods M5 comprises the coordinates of the 536 tips obtained again from the deposited inputs, the definitions of its eight reading paths and the conditions fixed before the reading, its readings and the distances it read, the check of the reader on synthetic coordinates, the check made before the reading on the deposited readout, as a published copy of the script with its output and a script that repeats it, the set-up and run records, and English translations of its contract and predictions, each with the SHA-256 digest of its Korean original. The chord link control of section 3.3 comprises six runs at the six seeds of the depth experiment, paired by seed with its runs at the clamp of 2.0, each with the per step traces of 500 records, the interval training log of 101 records, the summary, the end state of its adapter and the probe dumps of the 45 nodes at the start and the end; the smoke and gate runs made before them; the launch record with the hash of every code file; the reader fixed before the runs, its reading and the records of the local checks made before the runs; the reading made after the runs of the start and end states; the console; the seal with the hash of every file, the two records of the access stage, which name the account that held the model access token, being withheld; and English translations of the contract, the predictions and the kit supplement, each with the SHA-256 digest of its Korean original. The test of section 3.3 and Methods M17 that sets the Lorentz system against a flat one comprises 114 runs across 24 seeds, the Lorentz and flat systems at clamps of 10 and 2 at every seed and the cross entropy arm at the 18 seeds not run in the depth experiment, whose six cross entropy runs it reuses after one of them, run again, reproduced its deposited end state byte for byte; each run, the reproduction included, with the per step traces of 500 records, the interval training log of 101 records, the summary, the end state of its adapter and the probe dumps of the 45 nodes at the start and the end; the hidden states of the base at the common starting state for the 7,040 training samples and the calibration on them of the factor of the targets of each system; the launch record with the hash of every code file; the reader and the decider fixed before the runs, the reading, the decision and the records of the local checks made before the runs; the readings made after the runs; the seal with the hash of every file, the two records of the access stage, which name the account that held the model access token, being withheld; and English translations of the contract with its decision table, the tables fixed before the first stage on the GPU machine, the kit supplement and the amendment made before the runs, each with the SHA-256 digest of its Korean original.

These data are deposited at \url{https://doi.org/10.5281/zenodo.21438499}. Licences are assigned by file in \texttt{LICENSES.md} and \texttt{LICENSE\_MAP.csv}. Data and documentation are released under CC BY 4.0 except for EXAONE-derived outputs, which remain subject to the research-use terms of the EXAONE AI Model License Agreement 1.1 - NC, the end states of the paired runs of Methods M17.1, released, like the recipe's code, under Creative Commons Attribution-NonCommercial 4.0 International, and the LucaPhylo material specified below. Code assigned MIT in \texttt{LICENSE\_MAP.csv} is released under the MIT licence. The deposited LucaPhylo repository snapshot (\texttt{luca\_src\_pinned/}) and the pilot adapter and associated files in \texttt{stage4/} are distributed under the Apache License 2.0, as specified in \texttt{LICENSES.md} and \texttt{LICENSE\_MAP.csv}. The pilot is included in the deposit, with the coordinates, path definitions, reader and readings of the stored coordinate reading reported in Methods M5. Training corpora, model weights and tokenisers for the six systems of the operating registry are not included; they contain material under controlled access, and no conclusion in this manuscript depends on them. The per step production traces that Methods M1, M11, M12 and M16 quote are preserved with the original materials under restricted control, because they are operating records of production systems. Tables of these traces, with a data dictionary that states which fields are withheld, are available to the editors and reviewers of this manuscript on request through the editor, for confidential use in its assessment; requests are answered within five working days, and the tables may not be redistributed or reused beyond that assessment. After publication, readers can obtain the same tables and data dictionary on reasonable request, for verifying the reported analyses under a data use agreement that does not permit redistribution; these requests are also answered within five working days. The Korean originals of the documents fixed before the runs of Methods M5, M16.1, M17 and M17.1 and section 3.3 are available on the same terms. Requests should be addressed to K.S.S. (\href{mailto:sshin@csresearch.kr}{\nolinkurl{sshin@csresearch.kr}}).

\subsection{Code availability}\label{code-availability}

The reference implementation of the structural objectives and the projection, the stabilisation policy, the controlled comparison harness and the failure reproduction script are released under the MIT licence in the same deposit as the data. The reference implementation of the objectives depends only on a standard tensor library and requires no manifold package, so the arithmetic can be audited directly; the comparison harness and the failure reproduction script additionally use the transformer and adapter libraries whose versions are pinned in the deposit. Two lineages of the reference file exist in the author's records and differ in a curvature factor; the deposit carries the corrected lineage, and the discrepancy and its consequences are set out in Methods M11. The execution contract, launcher, environment check, preflight, instrumentation patch applied to the deposited harness, reader and self tests of the direct attempt of Methods M16.1, and the distance form switch, launcher, reader and self tests of its distance form comparison, and the capture wrapper, exact reader, structural check, launchers and radius reader of the capture rerun and of the radius observation, are released in the same deposit under the same licence. The protocol, runner, operations, reader with its exact arithmetic, seal, launcher, environment and access scripts and local tests of the paired stabiliser removal test of Methods M8 are released in the same deposit under the file-specific licences stated in \texttt{LICENSES.md} and \texttt{LICENSE\_MAP.csv}. The harness change, launcher, factor calculations and reader of the arms of Methods M16.1 that change the objective, with a script that makes a separate folder from which their published code can be launched again, the coordinate builder, path definitions and reader of the stored coordinate reading of Methods M5, and the chord link harness, runner, launcher, commands and reader of the chord link control of section 3.3 are released in the same deposit under the licences stated in \texttt{LICENSES.md} and \texttt{LICENSE\_MAP.csv}. The flat distance and the factor of the targets added to the harness, and the runner, plan, commands, capture, calibration and adoption of the targets, reproduction branch, launcher and seal of the test that sets the Lorentz system against a flat one, with its reader and decider, the scripts of the readings after the runs and of the reported numbers, and a verifier of the published copy, are released in the same deposit under the licences stated in \texttt{LICENSES.md} and \texttt{LICENSE\_MAP.csv}. The scripts that generate the Extended Data tables and draw the figures and Extended Data figures from the deposited records are released in the same deposit under the MIT licence; when the deposit was built, their tables matched the source from which those of this manuscript were set byte for byte, and their figures matched its figures pixel for pixel. The key of Fig. 1d was relabelled from ``this work, clamp shell'' to ``matched audit, clamp shell''; all plotted data and numerical values are unchanged from the deposited figure. Code for the measurements at pinned model revisions in Methods M2 is also released in the same deposit under the MIT licence. It includes the launchers, hidden-vector capture and comparison scripts, map replay scripts, exact stored-coordinate verifier and rerun preparation script. The launcher, protocol, harness patch, readers, seal and completion check of the controlled runs of section 2.6 and Methods M17, with a verifier of the published copy, the head and analysis code of the fixed representation experiment and of the readings after its runs, and the scripts of the 8 bit measurement and of the further calculations of Methods M1, M7 and M13 and Extended Data Table 1 are available in the same deposit under the file-specific licences stated in \texttt{LICENSES.md} and \texttt{LICENSE\_MAP.csv}. The arm switches and the training, probe, evaluation and reading scripts of the paired runs of Methods M17.1, with the scripts of the measurements made before the runs, of the arccosine bound counterfactual, of the recomputation of the curvature derivatives on a second GPU and of its repetition with the probe's own measurement, and of the calculations made from the records after the runs, are available in the same deposit under the same licences; they run on the code of ref. 28 at commit 6d443005, which the deposit includes under its own licence, Creative Commons Attribution-NonCommercial 4.0 International, as specified in \texttt{LICENSES.md} and \texttt{LICENSE\_MAP.csv}. For the six runs restarted from the released checkpoint, the deposit also holds the version of these scripts that ran them, which loads the released weights, checks the starting state and saves the state needed to resume, with its launcher, the scripts of the kernel diagnostic and the resume check made before the runs, and a script that recomputes from the deposited records every number reported for those runs, under the same licences. For the paired runs, the deposit records both hashes and the changes made to published script copies, including omitted checks against unpublished source documents. The comment on the bounded projection Jacobian in the reference implementation states the scope of that bound. The bound applies only to the projection Jacobian factor in the structural loss contribution; it does not by itself bound that contribution's full adapter gradient, and the cross entropy contribution bypasses the projection (Methods M4).

\subsection{Funding}\label{funding}

This work received no specific grant from any funding agency in the public, commercial or not for profit sectors.

\subsection{Ethics}\label{ethics}

This study involved no human or animal subjects.

\subsection{Correspondence}\label{correspondence}

Correspondence and requests for materials should be addressed to K.S.S.

\subsection{Acknowledgements}\label{acknowledgements}

Generative artificial intelligence systems were used as non authorial software tooling during the preparation of this manuscript, for analytical code drafting, manuscript drafting and revision, and computational cross checking. The analytical scripts are supplied with this submission together with the assertions and self tests they must pass, and the reported analyses are reproducible from the archived inputs and scripts independently of the tooling used to draft the code. The author reviewed the analytical scripts and reproduced the numerical results from the archived inputs. The research questions, design decisions, analytical choices, interpretations, conclusions, approval and responsibility for the content remain exclusively with the named human author. No generative artificial intelligence system is an author, contributor, creator or credit holder.

\section*{References}
\small
\begin{enumerate}\setlength{\itemsep}{1pt}
\item Nickel, M. \& Kiela, D. Poincaré embeddings for learning hierarchical representations. \emph{Adv. Neural Inf. Process. Syst.} \textbf{30} (2017). Preprint at \url{https://arxiv.org/abs/1705.08039}.
\item Sala, F., De Sa, C., Gu, A. \& Ré, C. Representation tradeoffs for hyperbolic embeddings. \emph{Proc. 35th Int. Conf. Mach. Learn.}, PMLR \textbf{80}, 4460--4469 (2018).
\item Yang, M. et al.~Hyperbolic fine-tuning for large language models. \emph{Adv. Neural Inf. Process. Syst.} \textbf{38} (2025). Preprint at \url{https://arxiv.org/abs/2410.04010}.
\item He, N. et al.~HELM: hyperbolic large language models via mixture-of-curvature experts. Preprint at \url{https://arxiv.org/abs/2505.24722} (2025).
\item Kim, J., Kim, E. \& Jang, D. Is the geometry doing the work? An operating-point audit of hierarchy in hyperbolic vision-language models. Preprint at \url{https://arxiv.org/abs/2607.05268} (2026).
\item Mishne, G., Wan, Z., Wang, Y. \& Yang, S. The numerical stability of hyperbolic representation learning. \emph{Proc. 40th Int. Conf. Mach. Learn.}, PMLR \textbf{202}, 24925--24949 (2023). Preprint at \url{https://arxiv.org/abs/2211.00181}.
\item Ermolov, A., Mirvakhabova, L., Khrulkov, V., Sebe, N. \& Oseledets, I. Hyperbolic vision transformers: combining improvements in metric learning. \emph{Proc. IEEE/CVF Conf. Comput. Vis. Pattern Recognit.} 7409--7419 (2022). Preprint at \url{https://arxiv.org/abs/2203.10833}.
\item Zhang, B. et al.~Parameter-efficient fine-tuning of LLMs with mixture of space experts. Preprint at \url{https://arxiv.org/abs/2602.14490} (2026).
\item Peng, Z., Xu, Z., Liu, Q., Yang, X. \& Shen, W. HyperET: efficient training in hyperbolic space for multi-modal large language models. \emph{Adv. Neural Inf. Process. Syst.} \textbf{38} (2025). Preprint at \url{https://arxiv.org/abs/2510.20322}.
\item Liu, Y. et al.~HyperGuide: hyperbolic guidance for efficient multi-step reasoning in large language models. Preprint at \url{https://arxiv.org/abs/2605.24140} (2026).
\item Ma, D., Li, Y., Ma, T. \& Pascoal, A. M. The state of the art in key technologies for autonomous underwater vehicles: a review. \emph{Engineering} (2025). \url{https://doi.org/10.1016/j.eng.2025.08.002}.
\item Tranzatto, M. et al.~CERBERUS in the DARPA Subterranean Challenge. \emph{Sci. Robot.} \textbf{7}, eabp9742 (2022).
\item Quattrini Li, A., Carver, C. J., Shao, Q., Zhou, X. \& Nelakuditi, S. Communication for underwater robots: recent trends. \emph{Curr. Robot. Rep.} \textbf{4}, 13--22 (2023).
\item Landon, L. B. et al.~\emph{Assessment of the State of Communication Delay Research in Preparation for Missions Beyond Low Earth Orbit.} NASA/TM--20250003885 (NASA Ames Research Center, 2025).
\item Merrill, W. \& Sabharwal, A. The parallelism tradeoff: limitations of log-precision transformers. \emph{Trans. Assoc. Comput. Linguist.} \textbf{11}, 531--545 (2023).
\item Merrill, W. \& Sabharwal, A. The expressive power of transformers with chain of thought. \emph{Int. Conf. Learn. Represent.} (2024).
\item Dziri, N. et al.~Faith and fate: limits of transformers on compositionality. \emph{Adv. Neural Inf. Process. Syst.} \textbf{36} (2023).
\item Balestriero, R., Pesenti, J. \& LeCun, Y. Learning in high dimension always amounts to extrapolation. Preprint at \url{https://arxiv.org/abs/2110.09485} (2021).
\item Polanyi, M. \emph{The Tacit Dimension} (Doubleday, 1966).
\item Shulman, L. S. Knowledge and teaching: foundations of the new reform. \emph{Harv. Educ. Rev.} \textbf{57}, 1--23 (1987).
\item Hooker, S. The hardware lottery. \emph{Commun. ACM} \textbf{64}, 58--65 (2021). Preprint at \url{https://arxiv.org/abs/2009.06489v2}.
\item Ganea, O.-E., Bécigneul, G. \& Hofmann, T. Hyperbolic neural networks. \emph{Adv. Neural Inf. Process. Syst.} \textbf{31} (2018). Preprint at \url{https://arxiv.org/abs/1805.09112}.
\item Bdeir, A. \& Landwehr, N. Rethinking the fully hyperbolic vision transformer in polar coordinates. Preprint at \url{https://arxiv.org/abs/2609.32899} (2026).
\item Yu, T. \& De Sa, C. Numerically accurate hyperbolic embeddings using tiling-based models. \emph{Adv. Neural Inf. Process. Syst.} \textbf{32} (2019).
\item Klein, T., Lang, T., Velaj, Y. \& Tschiatschek, S. hyperbolix: hyperbolic deep learning in JAX. Preprint at \url{https://arxiv.org/abs/2609.28248} (2026).
\item Celińska-Kopczyńska, D. \& Kopczyński, E. Numerical aspects of hyperbolic geometry. Preprint at \url{https://arxiv.org/abs/2404.09039} (2024).
\item Larroca, F., Bermolen, P., Perez Casulo, S., Fiori, M. \& Marenco, B. hypeGRL: hyperbolic graph representation learning library, \texttt{hypegrl/manifolds/polar.py}, commit e6ca276 (2026). \url{https://github.com/cicadaUY/hypeGRL}
\item Desai, K., Nickel, M., Rajpurohit, T., Johnson, J. \& Vedantam, R. Hyperbolic image-text representations. \emph{Proc. 40th Int. Conf. Mach. Learn.}, PMLR \textbf{202}, 7694--7731 (2023).
\item van Spengler, M., Berkhout, E. \& Mettes, P. Poincaré ResNet. \emph{Proc. IEEE/CVF Int. Conf. Comput. Vis.} (2023). Preprint at \url{https://arxiv.org/abs/2303.14027}.
\item Klein, T., Lang, T., Shkabrii, A. et al.~Understanding and improving hyperbolic deep reinforcement learning. \emph{Int. Conf. Learn. Represent.} (2026). Preprint at \url{https://arxiv.org/abs/2512.14202}.
\item Cetin, E., Chamberlain, B. P., Bronstein, M. M. \& Hunt, J. J. Hyperbolic deep reinforcement learning. Preprint at \url{https://arxiv.org/abs/2210.01542v1} (2022); published in \emph{Int. Conf. Learn. Represent.} (2023).
\item Angelakis, A. \& Gomez-Barrero, M. Curvature under attack in hZACH-ViT: gauge symmetry, boundary saturation, and adversarial failure. NeurReps 2026 (NeurIPS 2026 workshop). Preprint at \url{https://arxiv.org/abs/2610.00680} (2026).
\item Guo, Y., Wang, X., Chen, Y. \& Yu, S. X. Clipped hyperbolic classifiers are super-hyperbolic classifiers. \emph{Proc. IEEE/CVF Conf. Comput. Vis. Pattern Recognit.} (2022). Preprint at \url{https://arxiv.org/abs/2107.11472}.
\item He, Y., Yuan, Z., Chen, J. \& Horrocks, I. Language models as hierarchy encoders. \emph{Adv. Neural Inf. Process. Syst.} \textbf{37} (2024). Preprint at \url{https://arxiv.org/abs/2401.11374}.
\item Zhou, Z. et al.~Hyperbolic and evidence-prioritized experts for large vision-language models. Preprint at \url{https://arxiv.org/abs/2606.00275} (2026).
\item Bloom, B. S. (ed.) \emph{Taxonomy of Educational Objectives, Handbook I: Cognitive Domain} (Longmans, Green, 1956).
\item Anderson, L. W. \& Krathwohl, D. R. (eds) \emph{A Taxonomy for Learning, Teaching, and Assessing: A Revision of Bloom's Taxonomy of Educational Objectives} (Longman, 2001).
\item van der Klis, R. et al.~Fast and geometrically grounded Lorentz neural networks. Preprint at \url{https://arxiv.org/abs/2601.21529} (2026).
\item Pai, S. et al.~Foundation model for cancer imaging biomarkers. \emph{Nat. Mach. Intell.} \textbf{6}, 354--367 (2024).
\item Feng, B. et al.~A bioactivity foundation model using pairwise meta-learning. \emph{Nat. Mach. Intell.} \textbf{6}, 962--974 (2024).
\item He, Y. et al.~Generalized biological foundation model with unified nucleic acid and protein language. \emph{Nat. Mach. Intell.} \textbf{7}, 942--953 (2025).
\item Schmirler, R., Heinzinger, M. \& Rost, B. Fine-tuning protein language models boosts predictions across diverse tasks. \emph{Nat. Commun.} \textbf{15}, 7407 (2024).
\item Lewis, P. et al.~Retrieval-augmented generation for knowledge-intensive NLP tasks. \emph{Adv. Neural Inf. Process. Syst.} \textbf{33}, 9459--9474 (2020).
\item Yao, Y. et al.~Efficient GPT-4V level multimodal large language model for deployment on edge devices. \emph{Nat. Commun.} \textbf{16}, 5509 (2025).
\item Akiba, T., Shing, M., Tang, Y., Sun, Q. \& Ha, D. Evolutionary optimization of model merging recipes. \emph{Nat. Mach. Intell.} \textbf{7}, 195--204 (2025).
\item Favaro, L., Plehn, T., Qu, H. \& Spinner, J. Virtues and vices of equivariant transformers. Preprint at \url{https://arxiv.org/abs/2608.02735} (2026).
\item Kastaun, W., Kalinani, J. V. \& Ciolfi, R. Robust recovery of primitive variables in relativistic ideal magnetohydrodynamics. \emph{Phys. Rev.~D} \textbf{103}, 023018 (2021). Preprint at \url{https://arxiv.org/abs/2005.01821}.
\item Qwen Team. Qwen2.5 technical report. Preprint at \url{https://arxiv.org/abs/2412.15115} (2024).
\item An, S. et al.~EXAONE 3.5: series of large language models for real-world use cases. Preprint at \url{https://arxiv.org/abs/2412.04862} (2024).
\item Grattafiori, A. et al.~The Llama 3 herd of models. Preprint at \url{https://arxiv.org/abs/2407.21783} (2024).
\item Hu, E. J. et al.~LoRA: low-rank adaptation of large language models. \emph{Int. Conf. Learn. Represent.} (2022). Preprint at \url{https://arxiv.org/abs/2106.09685}.
\end{enumerate}
\normalsize
\clearpage
\section*{Figure legends}
\begin{figure}[H]\centering\includegraphics[width=\linewidth]{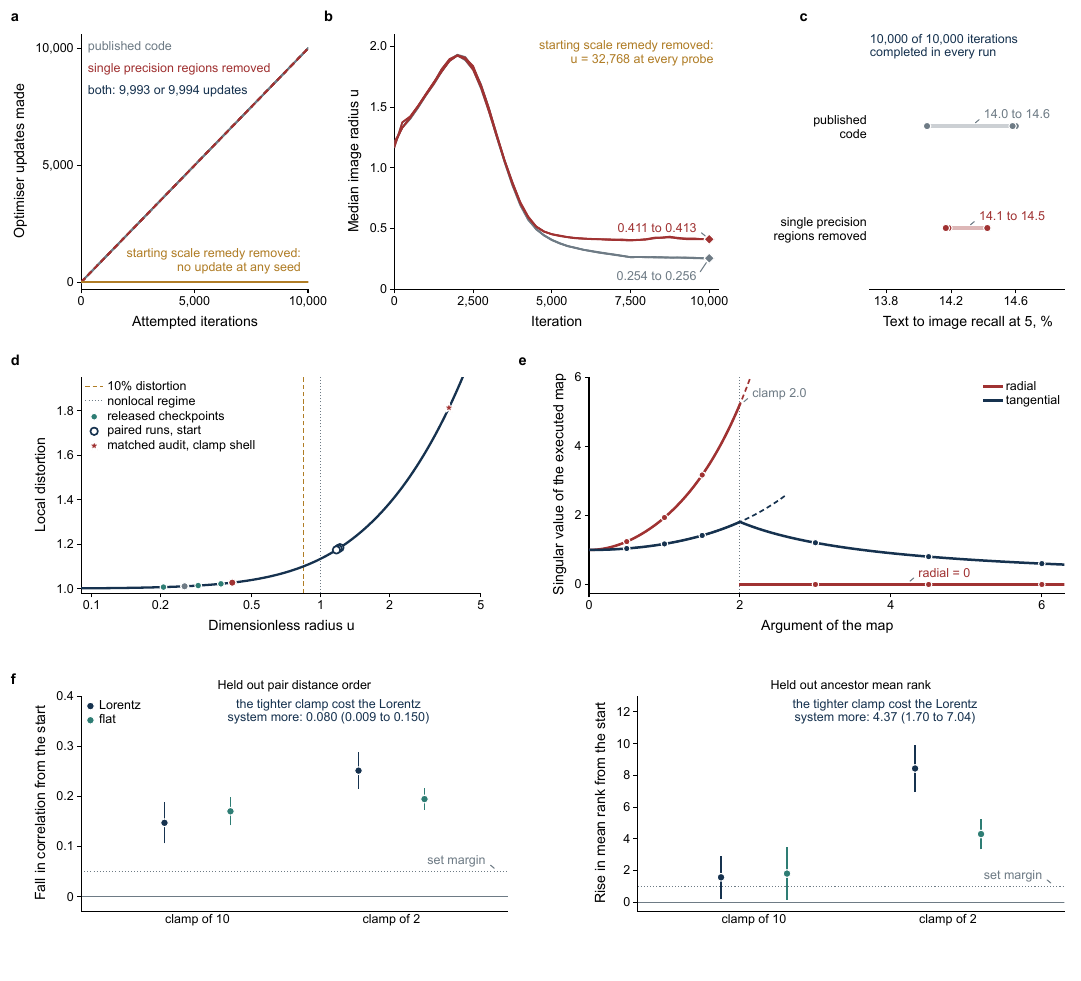}
\caption*{\textbf{Fig. 1 |} \textbf{Arithmetic and its remedies condition, alongside the objective, whether training proceeds and the states it reaches on the way.} \textbf{a}, \textbf{b}, \textbf{c}, Paired runs of a public hyperbolic image text recipe, objective fixed, three seeds per arm (Methods M17.1); arms named in \textbf{a}. \textbf{a}, Optimiser updates made against iterations attempted. \textbf{b}, Median image radius u on 256 held out pairs every 250 iterations (lines) and on 1,000 held out image embeddings at the end (diamonds); without the starting scale remedy the clamp bound every embedding. \textbf{c}, Text to image recall at 5 at the end; overlapping ranges are not a test of equivalence. \textbf{d}, Local distortion \texttt{H(u)\ =\ u\ /\ asinh(u)} against dimensionless radius in the audit's convention; filled, the paired runs at the end in the colours of \textbf{a}; star, the matched audit's clamp shell, \texttt{u\ =\ sinh(2)\ =\ 3.627}. \textbf{e}, Singular values of the exponential map executed after a clamp of 2.0: lines, closed forms (Methods M11); points, double precision finite differences; dashed, unclamped. Above the clamp the radial value is exactly zero. \textbf{f}, The prespecified test of the Lorentz system against a flat one (Methods M17): each system's change from the start in two held out read outs; means and paired 95\% t intervals over 24 seeds; dotted, the set margin. At the clamp of 10 the end states were not separated at the set margin; printed, the extra cost of the tighter clamp to the Lorentz system relative to the flat one.}\end{figure}
\begin{figure}[H]\centering\includegraphics[width=\linewidth]{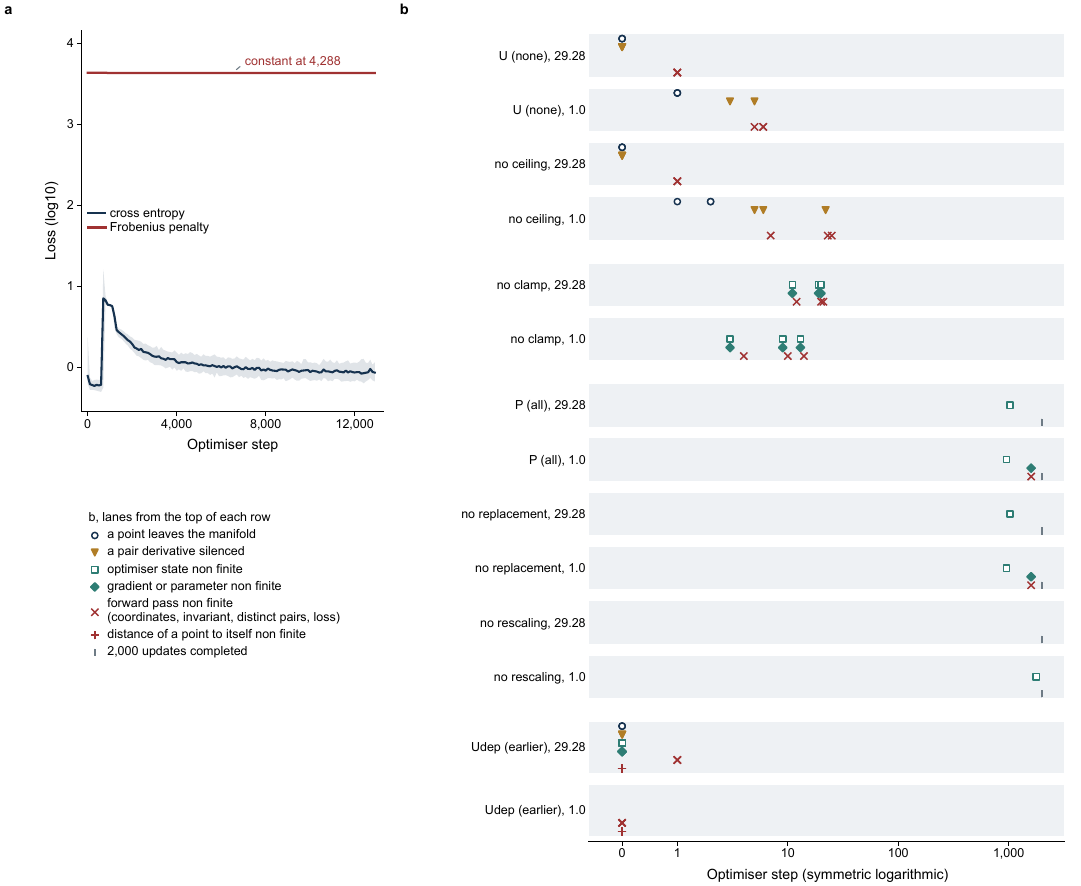}
\caption*{\textbf{Fig. 2 |} \textbf{Collapse and silence on the trainable path.} \textbf{a}, The one training trajectory retained in full, 236,323 logged records over 13,000 optimiser steps, logarithmic scale. Cross entropy (band, 10th to 90th percentile per 100 step interval) falls from 2.37 at the first record to a median of 0.86 over the final bin. The flat line above it, total loss minus cross entropy, is not a geometric term, since the hyperbolic tree path term was weighted zero: it is the Frobenius penalty on the adapter matrices, and both its constancy and its value, 4,288 where the same weights give 4,096.1, are bfloat16 arithmetic (Methods M12). No non finite value occurs. \textbf{b}, The order of events in each of the 42 runs of the paired stabiliser removal test (section 2.5; Methods M8; Extended Data Table 4). P keeps the four stabilisers, U none, each single removal all but the one named, and Udep runs the earlier deposited functions. Rows are arms and starts, lanes are kinds of event and the three seeds share the lanes; every marker sits at the step of its event, a bar marks 2,000 completed updates, and nothing is recorded after a non finite loss. The forward marker covers coordinates, the invariant, the arguments and distances of distinct points and the loss; the distance of a point to itself, evaluated only in Udep, is marked apart and was non finite there at step 0.}\end{figure}
\begin{figure}[H]\centering\includegraphics[width=\linewidth]{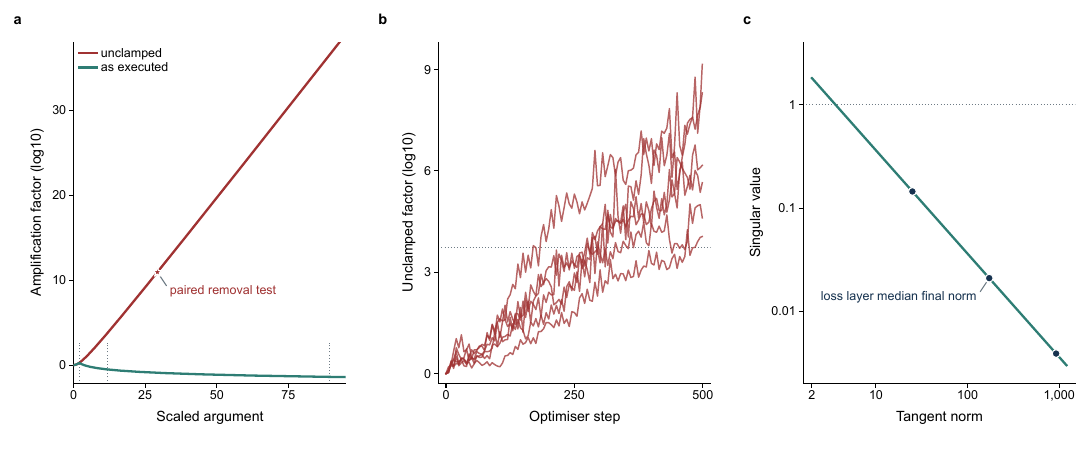}
\caption*{\textbf{Fig. 3 |} \textbf{The clamp: what it suppresses, and what it delivers instead.} \textbf{a}, Against the scaled argument, the curvature's square root times the tangent norm: upper curve, the unclamped amplification \texttt{sinh(θ)/θ}; lower curve, the tangential singular value of the executed map, equal to it below the clamp of 2.0 and \texttt{sinh(2)/θ}, the largest surviving singular value, above it. Dotted marks, the clamp and the half and single precision limits. Star, the higher target starting argument of the paired removal test of section 2.5, 29.3, where the unclamped factor is \numkeep{8.9 × 10¹⁰}. \textbf{b}, The unclamped factor over 500 steps in six seeds of the adapter manifold arm, all beginning at 1.0 and reaching a median of 959,576 and a maximum of \numkeep{1.44 × 10⁹}. Dotted rule, the factor at the half precision limit. The other arms record no such quantity: the loss layer arm's map receives the temperature scaled hidden state, not an adapter update, and the Euclidean and cross entropy arms evaluate no map. \textbf{c}, The same tangential singular value against the tangent norm, falling as its reciprocal. Points, automatic differentiation of the implementation at three norms; line, the closed form. The dotted rule at 1 separates amplification from contraction; at the loss layer arm's median final norm the map transmits 2.1\%. \textbf{Where the clamp is active, the radial singular value is exactly zero, and a logarithmic axis cannot show it.}}\end{figure}
\begin{figure}[H]\centering\includegraphics[width=\linewidth]{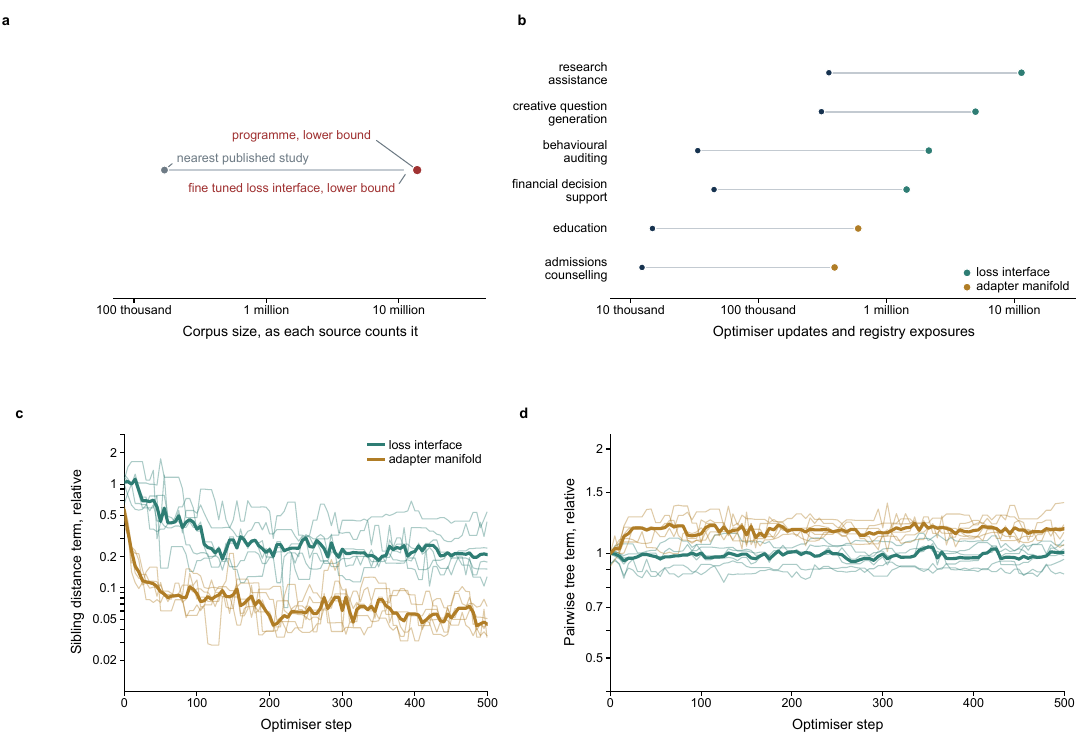}
\caption*{\textbf{Fig. 4 |} \textbf{The remaining placement is occupied at production scale.} \textbf{a}, Unique records in the retained input archives and source pools, on a logarithmic axis, counted within each system and summed without deduplicating across systems: at least 11,786,244 in the fine tuned systems at the loss interface and at least 13,787,131 in the programme, lower bounds because 118 of the research system's 1,068 round files were not read. The grey point is the larger of the two corpora used in the nearest published hyperbolic fine tuning of a language model, 170,000 samples in its own count. \textbf{b}, The six systems, each shown from its optimiser updates to its registry exposures, coloured by where the geometry is placed. Registry exposure counts the input examples of the recorded rounds and epochs, including repeated epochs and curriculum rounds, under the registry's original input and schedule accounting, and is a different quantity from the unique record count of \textbf{a}; the two are reported separately and never compared with one another. Registry totals are 20,951,552 exposures over 772,484 updates. The update count of the education system is a verified lower bound rather than an equality, because two rounds retain only partial logs. \textbf{c}, \textbf{d}, The two structural objectives in the matched audit of Methods M7, at the loss interface and on the adapter manifold, each seed relative to its first nonzero value over 500 steps, drawn as a 25 step running median, with the median across seeds in bold. \textbf{c}, The sibling distance term, which is angular, falls to a median of 0.21 at the loss interface. \textbf{d}, The pairwise tree term, which carries the depth weighting, does not move there, median 0.99. Under an active clamp both terms are angular at a fixed radius, so \textbf{c} is evidence that the geometry was not inert and \textbf{d} is an outcome of this audit.}\end{figure}
\begin{figure}[H]\centering\includegraphics[width=\linewidth]{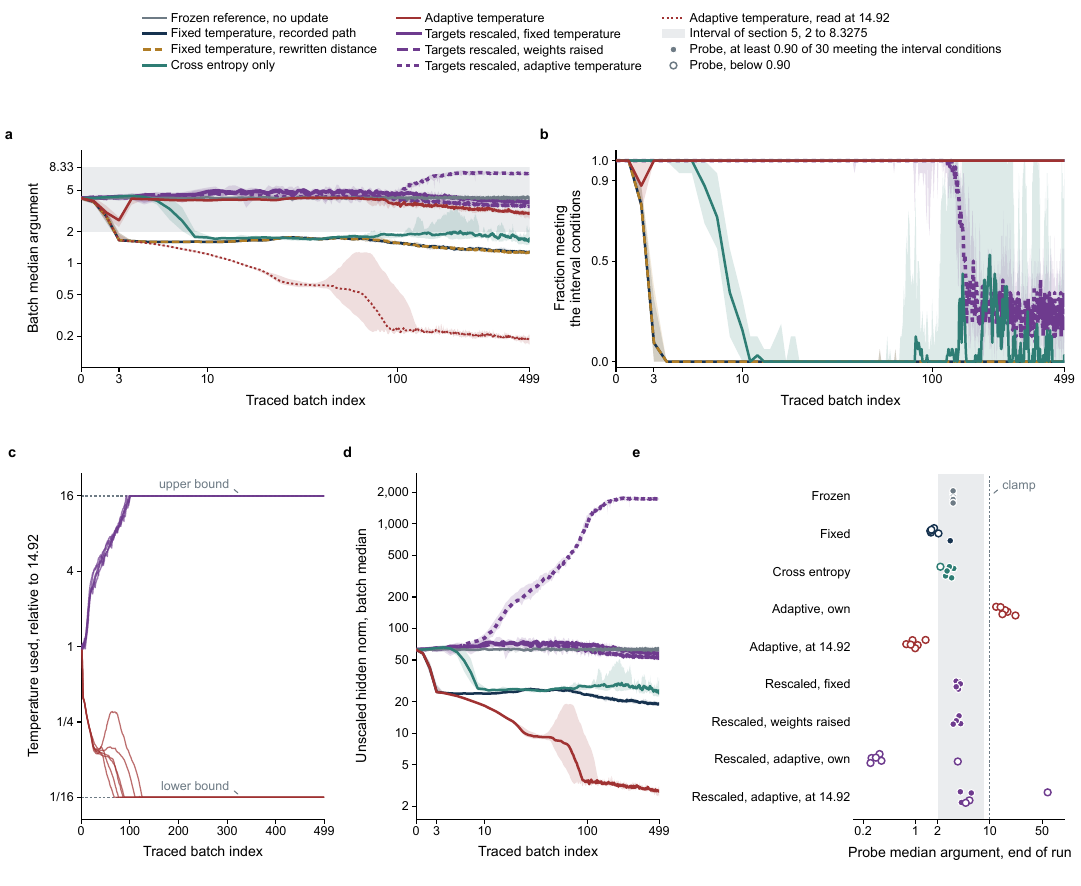}
\caption*{\textbf{Fig. 5 |} \textbf{The direct attempt of Methods M16.1 as one chain: entry, occupancy of the training batches, transfer to a fixed probe and the cost of control.} All arms use the lifted clamp of 10 and start at the projection temperature of 14.92, which adaptive arms set by feedback and others hold. Purple, the targets rescaled before the runs: solid, fixed temperature; dashed, structural weights also raised; dotted, adaptive temperature. Six seeds per arm, paired on one batch sequence; lines, across seed medians; bands, ranges. Horizontal axis, recorded batches, symmetric logarithmic in \textbf{a}, \textbf{b} and \textbf{d}; trained arms update after each. Batches hold 16 vectors, 11 at index 416. \textbf{a}, Batch median argument at the temperature in use; shaded, the interval of section 5. The rewritten distance (amber) lies on the recorded path; red dotted, the adaptive arm read at 14.92. \textbf{b}, Fraction of the batch meeting the interval conditions: argument inside the interval, logged residual within 0.01, clamp inactive. \textbf{c}, Controller temperature relative to 14.92, one line per seed. \textbf{d}, Median hidden norm before division by the temperature. \textbf{e}, The fixed probe of 30 sentences, 20 of seen and 10 of held out leaves, at the end of each run; filled, at least 0.90 of them met every interval condition. In every seed the red adaptive arm places the probe under the clamp at its own temperature and below the interval at 14.92: its batch occupancy did not transfer. Rewritten runs log no traces for \textbf{c} to \textbf{e}.}\end{figure}
\clearpage
\begin{table}[htbp]
\centering
\caption{\textbf{Extended Data Table 1 |} The invariant and range walls, per floating point format}
\label{tab:ed1}
\footnotesize\singlespacing
\setlength{\tabcolsep}{4pt}
\begin{tabularx}{\linewidth}{@{}Y r r Y r r r r@{}}
\toprule
Format & Significand bits & Invariant wall & How obtained & Analytic onset & Predicted band & Range wall & Width between them \\
\midrule
float64 & 53 & 18.7150 & closed form & 18.7150 & 18.3684 to 18.7150 & 710.4759 & 691.7609 \\
\addlinespace
float32 & 24 & 8.3275 & grid 0.0001 & 8.3178 & 8.3178 to 8.6643 & 89.4160 & 81.0885 \\
\addlinespace
tf32 & 11 & 4.1589 & closed form & 4.1589 & 3.8123 to 4.1589 & 89.4160 & 85.2571 \\
\addlinespace
float16 & 11 & 4.1589 & closed form & 4.1589 & 3.8123 to 4.1589 & 11.7830 & 7.6241 \\
\addlinespace
bfloat16 & 8 & 2.7804 & grid 0.0001 & 2.7803 & 2.7726 to 3.1192 & 89.4121 & 86.6317 \\
\addlinespace
fp8 e4m3 & 4 & 1.4982 & grid 0.0001 & 1.4982 & 1.3863 to 1.7329 & 6.7979 & 5.2997 \\
\addlinespace
fp8 e5m2 & 3 & 1.3863 & closed form & 1.3863 & 1.0397 to 1.3863 & 11.6500 & 10.2637 \\
\addlinespace
fp6 e3m2 & 3 & 1.3863 & closed form & 1.3863 & 1.0397 to 1.3863 & 4.0250 & 2.6387 \\
\addlinespace
fp6 e2m3 & 4 & 1.4982 & grid 0.0001 & 1.4982 & 1.3863 to 1.7329 & 2.7036 & 1.2054 \\
\addlinespace
fp4 e2m1 & 2 & 1.0476 & grid 0.0001 & 1.0476 & 0.6931 to 1.0397 & 2.4779 & 1.4303 \\
\bottomrule
\end{tabularx}
\begin{flushleft}\footnotesize
The invariant wall is the argument from which the two Lorentz coordinates round to one stored number over an interval of arguments, and the range wall is the argument at which the time coordinate stops being representable. Each depends on one half of the format and on nothing else: the invariant wall on the significand, through the relative gap 2 / (exp(2t) + 1), which is 2 exp(-2t) beyond the smallest formats, falling below the local spacing, and the range wall on the largest finite value, as acosh of it. The predicted band runs from p ln2 / 2 to (p + 1) ln2 / 2, a width of ln2 / 2 that is the factor of two between the two ends of a binade; it is derived rather than fitted, and every onset in this table for a significand of 3 bits or more lies in it, endpoints included, and the grid onsets of Methods M1 agree with it to within their grid of 0.001, the half precision value there sitting 0.0001 above the exact endpoint. The band endpoints are the closed forms p ln2 / 2 and (p + 1) ln2 / 2, shown to 4 decimals rounded to nearest, so that 2 ln2 = 1.386294 appears as 1.3863 both as the upper end for p = 3 and as the lower end for p = 4. Where the significand is odd in width the onset is the upper endpoint in closed form, because the exponential is a power of two there and the space coordinate lands on the midpoint below it, which round to nearest with ties to even carries up to the same stored number as the time coordinate; the signature is a gap of exactly half a spacing, and float64, float16, tf32 and the 5 bit exponent 8 bit float all show it. The even significands, float32, bfloat16 and the 4 bit exponent 8 bit float, sit strictly inside. The fp6 and fp4 rows are the element encodings of the Open Compute Project microscaling (MX) formats at unit scale, with largest finite values 28, 7.5 and 6, evaluated under the same rounding as the other rows, and no fp6 or fp4 kernel was run: fp6 e3m2 shows the same signature, fp6 e2m3 sits strictly inside, and the 2 bit significand of fp4 e2m1 places its first interval of coincidence past the band's upper end. Before that interval the fp4 coordinates coincide at one argument only, ln 2, the band's lower end, where they are exactly 1.25 and 0.75 and both tie to the smallest normal, 1, whose rounding cell the subnormal 0.5 makes symmetric; that is the coincidence at the smallest normal of Methods M1, which needs that normal to be at least the value B defined there; among these formats only fp4 e2m1 reaches it, with equality, and no grid records the point. Some entries are consequences the law states before the row is computed, under the same convention as the rest. tf32 carries float32 exponent with float16 significand and therefore inherits one wall from each, which it does exactly; a kernel that accumulates or stores in float32 is a different path, which Methods M1 measures for the 8 bit encodings. Both 8 bit encodings reach their first coincidence below 2.0, the clamp used in this work, at every per tensor scale at which the coordinates first coincide at normal values (Methods M1), so the argument the clamp saturates to lies past the onset; there the stored scalar invariant is 0 in the encoding with a 4 bit significand and 3.75 in the other instead of 1, and a saturated vector meets a stored invariant tolerance of 0.01 only in exceptional directions (Methods M1). This table gives the exact onset and the onset observed on a grid, not the value a library returns. The odd significand rows are the exact onset for continuously varying real arguments under correct rounding, given by the closed form and involving no evaluation of either function; they are displayed to 4 decimal places, so a displayed value may round up past the endpoint it equals. The even significand rows are the first coincidence observed on the stated grid, with each coordinate evaluated in double precision and then converted to the format, a conversion that equalled exact single rounding at every coordinate checked on the GPU and gave the same first coincidences on the CPU (Methods M1). The analytic onset is where coincidence for continuous input first holds on an interval of positive length: for an odd significand the closed form above, attained, and for the even ones the infimum of the first interval of coincidence, which is not attained, because at that argument the space coordinate lies on a midpoint that ties to even carry down while the time coordinate rounds up; in single precision that interval is 1.4 times ten to the minus fourteenth wide, which is why the grid first records it at 8.3275. What a library returns is reported in Methods M1 and differs between implementations. The width column is the span from the invariant wall to the range wall, below which the time coordinate is representable; past the first coincidence, stored equality and a 1 bit difference alternate (Methods M1).
\end{flushleft}
\end{table}

\begin{table}[htbp]
\centering
\caption{\textbf{Extended Data Table 2 |} Hidden state norms at the projection input, and the Lorentz identity they produce}
\label{tab:ed2}
\footnotesize\singlespacing
\setlength{\tabcolsep}{4pt}
\begin{tabularx}{\linewidth}{@{}Y r r r r r r r r@{}}
\toprule
Model & Layer & n & Median & p90 & p99 & Max & > 89.4160 & > 11.7830 \\
\midrule
Qwen2.5-7B & 0 & 36 & 0.95 & 1.00 & 1.10 & 1.10 & 0\% & 0\% \\
Qwen2.5-7B & 7 & 36 & 40.80 & 48.26 & 52.01 & 52.71 & 0\% & 100\% \\
Qwen2.5-7B & 14 & 36 & 70.87 & 75.78 & 78.57 & 79.51 & 0\% & 100\% \\
Qwen2.5-7B & 21 & 36 & 130.05 & 136.53 & 142.89 & 145.35 & 100\% & 100\% \\
Qwen2.5-7B & 28 & 36 & 547.28 & 619.38 & 690.76 & 695.59 & 100\% & 100\% \\
\textbf{Qwen2.5-7B} & \textbf{all} & 180 & 70.87 & 546.75 & 665.48 & 695.59 & 40\% & 80\% \\
\addlinespace
EXAONE-3.5-7.8B & 0 & 36 & 0.49 & 0.66 & 0.79 & 0.80 & 0\% & 0\% \\
EXAONE-3.5-7.8B & 8 & 36 & 3.17 & 3.55 & 3.78 & 3.86 & 0\% & 0\% \\
EXAONE-3.5-7.8B & 16 & 36 & 7.22 & 8.31 & 8.85 & 9.04 & 0\% & 0\% \\
EXAONE-3.5-7.8B & 24 & 36 & 19.00 & 21.66 & 23.04 & 23.30 & 0\% & 100\% \\
EXAONE-3.5-7.8B & 32 & 36 & 56.12 & 67.58 & 73.04 & 73.80 & 0\% & 100\% \\
\textbf{EXAONE-3.5-7.8B} & \textbf{all} & 180 & 7.22 & 55.52 & 69.73 & 73.80 & 0\% & 40\% \\
\addlinespace
Llama-3.1-8B & 0 & 36 & 0.54 & 0.71 & 0.78 & 0.79 & 0\% & 0\% \\
Llama-3.1-8B & 8 & 36 & 6.17 & 6.87 & 7.41 & 7.43 & 0\% & 0\% \\
Llama-3.1-8B & 16 & 36 & 11.81 & 12.93 & 13.23 & 13.26 & 0\% & 53\% \\
Llama-3.1-8B & 24 & 36 & 28.09 & 29.95 & 30.51 & 30.58 & 0\% & 100\% \\
Llama-3.1-8B & 32 & 36 & 74.02 & 83.99 & 89.74 & 90.61 & 3\% & 100\% \\
\textbf{Llama-3.1-8B} & \textbf{all} & 180 & 11.81 & 73.60 & 87.40 & 90.61 & 1\% & 51\% \\
\bottomrule
\end{tabularx}
\par\vspace{10pt}
\begin{tabularx}{\linewidth}{@{}Y r r r r r r@{}}
\toprule
Model & Layer & id err f64 & id err f32 & id err bf16 & id err f16 & id lost bf16 \\
\midrule
Qwen2.5-7B & 0 & 0.00e+00 & 1.64e-07 & 1.01e-02 & 7.25e-04 & 0\% \\
Qwen2.5-7B & 7 & \textbf{1.00} & \textbf{1.00} & \textbf{1.00} & overflow & 100\% \\
Qwen2.5-7B & 14 & \textbf{1.00} & \textbf{1.00} & \textbf{1.00} & overflow & 100\% \\
Qwen2.5-7B & 21 & \textbf{1.00} & overflow & overflow & overflow & 100\% \\
Qwen2.5-7B & 28 & overflow & overflow & overflow & overflow & 100\% \\
\textbf{Qwen2.5-7B} & \textbf{all} & \textbf{1.00} & \textbf{1.00} & \textbf{1.00} & overflow & 80\% \\
\addlinespace
EXAONE-3.5-7.8B & 0 & 0.00e+00 & 1.27e-08 & 9.77e-03 & 3.15e-05 & 0\% \\
EXAONE-3.5-7.8B & 8 & 0.00e+00 & 8.30e-06 & 4.88e-01 & 6.83e-02 & 44\% \\
EXAONE-3.5-7.8B & 16 & 0.00e+00 & 4.00e-03 & \textbf{1.00} & \textbf{1.00} & 100\% \\
EXAONE-3.5-7.8B & 24 & \textbf{1.00} & \textbf{1.00} & \textbf{1.00} & overflow & 100\% \\
EXAONE-3.5-7.8B & 32 & \textbf{1.00} & \textbf{1.00} & \textbf{1.00} & overflow & 100\% \\
\textbf{EXAONE-3.5-7.8B} & \textbf{all} & 0.00e+00 & 4.00e-03 & \textbf{1.00} & \textbf{1.00} & 69\% \\
\addlinespace
Llama-3.1-8B & 0 & 0.00e+00 & 6.35e-08 & 1.91e-03 & 8.02e-04 & 0\% \\
Llama-3.1-8B & 8 & 0.00e+00 & 7.81e-04 & \textbf{1.00} & \textbf{1.00} & 100\% \\
Llama-3.1-8B & 16 & 9.54e-07 & \textbf{1.00} & \textbf{1.00} & overflow & 100\% \\
Llama-3.1-8B & 24 & \textbf{1.00} & \textbf{1.00} & \textbf{1.00} & overflow & 100\% \\
Llama-3.1-8B & 32 & \textbf{1.00} & \textbf{1.00} & \textbf{1.00} & overflow & 100\% \\
\textbf{Llama-3.1-8B} & \textbf{all} & 9.54e-07 & \textbf{1.00} & \textbf{1.00} & overflow & 80\% \\
\bottomrule
\end{tabularx}
\begin{flushleft}\footnotesize
One record per sequence per depth, taking the last non padding token, which is the position the projection of Methods M3 receives. No training, adapter or projection is involved. The two count columns are the range walls of Methods M1. The four identity columns give \texttt{|cosh\^{}2 - sinh\^{}2 - 1|} evaluated at the median norm with the two coordinates stored in the named format; \textbf{1.00} means the coordinates are the same stored number and the point is not on the manifold. The final column gives the percentage of individual representations at that depth for which the two coordinates are one stored number in bfloat16, the format in which the matched audit loads its base model; that audit computes its geometry in float32. At layer 0 no representation reaches a norm of 1.11 and the identity holds in every format.
\end{flushleft}
\end{table}

\begin{table}[htbp]
\centering
\caption{\textbf{Extended Data Table 3 |} The same measurement over all token positions}
\label{tab:ed3}
\footnotesize\singlespacing
\setlength{\tabcolsep}{4pt}
\begin{tabularx}{\linewidth}{@{}Y r r r r r r r r@{}}
\toprule
Model & Layer & n & Median & p90 & p99 & Max & > 89.4160 & > 11.7830 \\
\midrule
Qwen2.5-7B & 0 & 438 & 0.95 & 0.99 & 1.10 & 1.10 & 0\% & 0\% \\
Qwen2.5-7B & 7 & 438 & 44.98 & 154.15 & 1.43e4 & 1.47e4 & 11\% & 100\% \\
Qwen2.5-7B & 14 & 438 & 72.94 & 181.06 & 1.53e4 & 1.57e4 & 14\% & 100\% \\
Qwen2.5-7B & 21 & 438 & 134.34 & 278.72 & 1.53e4 & 1.56e4 & 100\% & 100\% \\
Qwen2.5-7B & 28 & 438 & 564.24 & 812.10 & 2,563.41 & 2,563.41 & 100\% & 100\% \\
\textbf{Qwen2.5-7B} & \textbf{all} & 2,190 & 75.15 & 625.56 & 1.49e4 & 1.57e4 & 45\% & 80\% \\
\addlinespace
EXAONE-3.5-7.8B & 0 & 447 & 0.52 & 0.70 & 0.77 & 0.80 & 0\% & 0\% \\
EXAONE-3.5-7.8B & 8 & 447 & 3.32 & 6.30 & 745.90 & 745.90 & 8\% & 8\% \\
EXAONE-3.5-7.8B & 16 & 447 & 7.39 & 162.75 & 745.66 & 745.66 & 16\% & 16\% \\
EXAONE-3.5-7.8B & 24 & 447 & 19.06 & 159.12 & 746.69 & 746.69 & 16\% & 98\% \\
EXAONE-3.5-7.8B & 32 & 447 & 49.42 & 62.55 & 72.25 & 81.34 & 0\% & 100\% \\
\textbf{EXAONE-3.5-7.8B} & \textbf{all} & 2,235 & 7.56 & 62.69 & 745.90 & 746.69 & 8\% & 45\% \\
\addlinespace
Llama-3.1-8B & 0 & 474 & 0.57 & 0.72 & 0.78 & 0.82 & 0\% & 0\% \\
Llama-3.1-8B & 8 & 474 & 6.37 & 8.69 & 531.07 & 531.07 & 8\% & 8\% \\
Llama-3.1-8B & 16 & 474 & 11.68 & 12.89 & 534.46 & 534.46 & 8\% & 46\% \\
Llama-3.1-8B & 24 & 474 & 27.12 & 30.68 & 534.27 & 534.27 & 8\% & 100\% \\
Llama-3.1-8B & 32 & 474 & 71.62 & 101.22 & 104.04 & 181.93 & 14\% & 100\% \\
\textbf{Llama-3.1-8B} & \textbf{all} & 2,370 & 11.91 & 79.09 & 534.46 & 534.46 & 7\% & 51\% \\
\bottomrule
\end{tabularx}
\par\vspace{10pt}
\begin{tabularx}{\linewidth}{@{}Y r r r r r r@{}}
\toprule
Model & Layer & id err f64 & id err f32 & id err bf16 & id err f16 & id lost bf16 \\
\midrule
Qwen2.5-7B & 0 & 4.44e-16 & 2.06e-07 & 1.01e-02 & 3.05e-05 & 0\% \\
Qwen2.5-7B & 7 & \textbf{1.00} & \textbf{1.00} & \textbf{1.00} & overflow & 100\% \\
Qwen2.5-7B & 14 & \textbf{1.00} & \textbf{1.00} & \textbf{1.00} & overflow & 100\% \\
Qwen2.5-7B & 21 & \textbf{1.00} & overflow & overflow & overflow & 100\% \\
Qwen2.5-7B & 28 & overflow & overflow & overflow & overflow & 100\% \\
\textbf{Qwen2.5-7B} & \textbf{all} & \textbf{1.00} & \textbf{1.00} & \textbf{1.00} & overflow & 80\% \\
\addlinespace
EXAONE-3.5-7.8B & 0 & 2.22e-16 & 4.11e-08 & 1.95e-03 & 2.74e-04 & 0\% \\
EXAONE-3.5-7.8B & 8 & 2.84e-14 & 2.07e-05 & 7.38e-01 & 8.55e-02 & 53\% \\
EXAONE-3.5-7.8B & 16 & 1.16e-10 & 9.12e-02 & \textbf{1.00} & \textbf{1.00} & 100\% \\
EXAONE-3.5-7.8B & 24 & \textbf{1.00} & \textbf{1.00} & \textbf{1.00} & overflow & 100\% \\
EXAONE-3.5-7.8B & 32 & \textbf{1.00} & \textbf{1.00} & \textbf{1.00} & overflow & 100\% \\
\textbf{EXAONE-3.5-7.8B} & \textbf{all} & 2.33e-10 & 5.36e-02 & \textbf{1.00} & \textbf{1.00} & 71\% \\
\addlinespace
Llama-3.1-8B & 0 & 2.22e-16 & 1.41e-07 & 6.84e-03 & 6.68e-06 & 0\% \\
Llama-3.1-8B & 8 & 1.46e-11 & 1.30e-03 & \textbf{1.00} & \textbf{1.00} & 100\% \\
Llama-3.1-8B & 16 & 9.54e-07 & \textbf{1.00} & \textbf{1.00} & \textbf{1.00} & 100\% \\
Llama-3.1-8B & 24 & \textbf{1.00} & \textbf{1.00} & \textbf{1.00} & overflow & 100\% \\
Llama-3.1-8B & 32 & \textbf{1.00} & \textbf{1.00} & \textbf{1.00} & overflow & 100\% \\
\textbf{Llama-3.1-8B} & \textbf{all} & 9.54e-07 & \textbf{1.00} & \textbf{1.00} & overflow & 80\% \\
\bottomrule
\end{tabularx}
\begin{flushleft}\footnotesize
The same records without the last token selection, 6,795 in all across the three models; Extended Data Fig. 1b is drawn from the 4,425 records of the first two. The maximum for the first model is 15,668 against 696 at the last token, and the difference is carried by positions 0 and 1, so the upper tail over all positions describes the leading special tokens rather than the representation the projection receives, while the per layer medians of the two populations differ by at most 14\% (section 2.2). The identity values of section 2.2 are those of Extended Data Table 2.
\end{flushleft}
\end{table}

\begin{table}[htbp]
\centering
\caption{\textbf{Extended Data Table 4 |} The paired stabiliser removal test, per run: the order of events}
\label{tab:ed4}
\footnotesize\singlespacing
\setlength{\tabcolsep}{4pt}
\begin{tabularx}{\linewidth}{@{}Y r Y Y Y Y Y Y Y@{}}
\toprule
Arm & Start & First out & All out & First silenced & Optimiser & Gradient & Parameter & Forward \\
\midrule
P and no replacement & 29.28 & none & none & none & 1,026 / none / none & none & none & none \\
\addlinespace
P and no replacement & 1.0 & c 1,597 / none / none & c 1,597 / none / none & c 1,597 / none / none & 957 / none / none & 1,596 / none / none & 1,596 / none / none & 1,597 / none / none \\
\addlinespace
U & 29.28 & 0 & 0 & 0 & c 1 & c 1 & c 1 & 1 \\
\addlinespace
U & 1.0 & 1 & 2 / 1 / 1 & 3 / 5 / c 6 & c 5 / c 6 / c 6 & c 5 / c 6 / c 6 & c 5 / c 6 / c 6 & 5 / 6 / 6 \\
\addlinespace
Udep & 29.28 & 0 & 0 & 0 & 0 & 0 & 0 & 1 \\
\addlinespace
Udep & 1.0 & c 0 & c 0 & c 0 & c 0 & c 0 & c 0 & 0 \\
\addlinespace
no ceiling & 29.28 & 0 & 0 & 0 & c 1 & c 1 & c 1 & 1 \\
\addlinespace
no ceiling & 1.0 & 2 / 2 / 1 & 3 / 2 / 2 & 22 / 6 / 5 & c 23 / c 25 / c 7 & c 23 / c 25 / c 7 & c 23 / c 25 / c 7 & 23 / 25 / 7 \\
\addlinespace
no clamp & 29.28 & c 12 / c 20 / c 21 & c 12 / c 20 / c 21 & c 12 / c 20 / c 21 & 11 / 19 / 20 & 11 / 19 / 20 & 11 / 19 / 20 & 12 / 20 / 21 \\
\addlinespace
no clamp & 1.0 & c 4 / c 10 / c 14 & c 4 / c 10 / c 14 & c 4 / c 10 / c 14 & 3 / 9 / 13 & 3 / 9 / 13 & 3 / 9 / 13 & 4 / 10 / 14 \\
\addlinespace
no rescaling & 29.28 & none & none & none & none & none & none & none \\
\addlinespace
no rescaling & 1.0 & none & none & none & none / 1,775 / none & none & none & none \\
\bottomrule
\end{tabularx}
\begin{flushleft}\footnotesize
Each row holds the three seeds of an arm and start, seeds 0, 1 and 2 in that order, with one value where all three agree. P carries the stabilisers of Methods M8, U carries none, and Udep runs unchanged the functions deposited with the earlier, unpaired run of this route; each single removal carries all but one of the stabilisers and lacks the one named: the ceiling on the scaled argument, the clamp of the inverse hyperbolic cosine, the replacement of non finite map outputs or the rescaling of the routed second factor. P and the arm without the replacement are equal in every entry of every seed and share a row. Start is the target batch mean scaled argument at step 0. First out and all out are the first steps at which a point, and every point, had left the manifold; first silenced is the first step at which a pair argument was silenced; optimiser, gradient and parameter are the first steps at which the optimiser state, a gradient and a parameter held a non finite value; forward is the first step at which a coordinate, the invariant, a pair argument or distance between distinct points, or the total loss was non finite. The distance of each point to itself is recorded apart: in Udep at 29.28 its argument, 1 in exact arithmetic, fell below 1 at step 0 in 10, 12 and 9 of the 16 self pairs (per step trace), where the inverse hyperbolic cosine is undefined, while the distances between distinct points and the loss were finite and the gradients, optimiser states and parameters were already non finite. In these event columns none marks an event absent from a run that completed its 2,000 updates, and c k one absent from a run that stopped at a non finite total loss at step k, for which the record says nothing after k. The entries from First out to Forward are fields of the per run summary that the deposited runner writes, set out without recomputation.
\end{flushleft}
\end{table}

\begin{table}[htbp]
\centering
\caption{\textbf{Extended Data Table 5 |} The paired stabiliser removal test, per run: outcomes}
\label{tab:ed5}
\footnotesize\singlespacing
\setlength{\tabcolsep}{4pt}
\begin{tabularx}{\linewidth}{@{}Z{1.00} r Z{0.80} Z{1.75} Z{0.80} r r Z{1.00} Z{0.65}@{}}
\toprule
Arm & Start & Updates & First non finite & Step & Members & Silenced & Largest composite & First above \\
\midrule
P and no replacement & 29.28 & 2,000 & backward optimiser state / none / none & 1,026 / none / none & 16 & 0 & 4.98 at 944 / 3.67 at 0 / 11.27 at 10 & none / none / 10 \\
\addlinespace
P and no replacement & 1.0 & 1,597 / 2,000 / 2,000 & backward optimiser state / none / none & 957 / none / none & 16 & 0 & 12.92 at 4 / 6.60 at 4 / 4.88 at 14 & 4 / 4 / none \\
\addlinespace
U & 29.28 & 1 & forward coordinates / forward pair arguments / forward pair arguments & 1 & 0 & 236 / 240 / 240 & 3.38e37 at 1 / 2.91e22 at 1 / 6.72e21 at 1 & 0 \\
\addlinespace
U & 1.0 & 5 / 6 / 6 & forward coordinates / forward pair arguments / forward pair arguments & 5 / 6 / 6 & 16 & 0 & 3.64e13 at 4 / 1.90e27 at 6 / 4.13e19 at 6 & 1 \\
\addlinespace
Udep & 29.28 & 1 & backward self pair distances & 0 & 0 & 160 / 168 / 198 & 4.39e15 at 0 / 8.40e15 at 0 / 6.59e15 at 0 & 0 \\
\addlinespace
Udep & 1.0 & 0 & forward distances & 0 & 16 & 0 & 0.0428 at 0 / 0.0598 at 0 / 0.0628 at 0 & none \\
\addlinespace
no ceiling & 29.28 & 1 & forward coordinates / forward pair arguments / forward pair arguments & 1 & 0 & 236 / 240 / 240 & 3.38e37 at 1 / 5.10e22 at 1 / 5.75e21 at 1 & 0 \\
\addlinespace
no ceiling & 1.0 & 23 / 25 / 7 & forward pair arguments & 23 / 25 / 7 & 16 & 0 & 3.95e20 at 23 / 1.45e32 at 25 / 2.85e20 at 7 & 1 \\
\addlinespace
no clamp & 29.28 & 12 / 20 / 21 & backward gradients & 11 / 19 / 20 & 16 & 0 & 2.79 at 0 / 3.67 at 0 / 10.21 at 10 & none / none / 10 \\
\addlinespace
no clamp & 1.0 & 4 / 10 / 14 & backward gradients & 3 / 9 / 13 & 16 & 0 & 4.70 at 1 / 6.60 at 4 / 1.99 at 0 & none / 4 / none \\
\addlinespace
no rescaling & 29.28 & 2,000 & none & none & 16 & 0 & 3.43 at 67 / 3.67 at 0 / 3.54 at 0 & none \\
\addlinespace
no rescaling & 1.0 & 2,000 & none / backward optimiser state / none & none / 1,775 / none & 16 & 0 & 8.90 at 5 / 1.89 at 0 / 1.99 at 0 & 4 / none / none \\
\bottomrule
\end{tabularx}
\begin{flushleft}\footnotesize
Rows, arms and starts are as in Extended Data Table 4. Updates counts completed optimiser updates, and a run with fewer than 2,000 ended at a non finite total loss at the step equal to its number of updates. First non finite names the pass and the quantity in which a non finite value first appeared, and Step gives its step; the optimiser state is updated after the backward pass and is listed with it, and none means that no non finite value appeared. Members counts, of the 16 points of the batch, those whose stored invariant lies within 0.01 of 1 at step 0, read exactly from the stored coordinates in the invariant of the run's lineage; Silenced counts, of the 240 pair arguments the loss uses, those that are finite and at or above 2⁶⁴ at step 0. Largest composite is the largest value of the composite derivative of Methods M8 over every forward pass, with its step, and First above is the first step at which a batch row exceeded the constant 5.22573 of Methods M4. The entries from Start to First above are fields of the per run record that the deposited reader writes, set out without recomputation.
\end{flushleft}
\end{table}

\begin{table}[htbp]
\centering
\caption{\textbf{Extended Data Table 6 |} The matched four arm audit, per seed}
\label{tab:ed6}
\footnotesize\singlespacing
\setlength{\tabcolsep}{4pt}
\begin{tabularx}{\linewidth}{@{}Y r r r r r r r r@{}}
\toprule
Arm & Seed & Steps & Non finite & Tangent norm & Manifold departure & Unclamped factor & Rank corr. & MRR \\
\midrule
ce & 7 & 500 & 0 & n/a & n/a & -- & 0.2358 & 0.5536 \\
\addlinespace
ce & 42 & 500 & 0 & n/a & n/a & -- & 0.2613 & 0.7091 \\
\addlinespace
ce & 99 & 500 & 0 & n/a & n/a & -- & 0.2551 & 0.8250 \\
\addlinespace
ce & 123 & 500 & 0 & n/a & n/a & -- & 0.2447 & 0.8333 \\
\addlinespace
ce & 777 & 500 & 0 & n/a & n/a & -- & 0.2119 & 0.7933 \\
\addlinespace
ce & 2024 & 500 & 0 & n/a & n/a & -- & 0.3114 & 0.8833 \\
\addlinespace
euclid & 7 & 500 & 0 & 161.93 & n/a & -- & 0.2494 & 0.6341 \\
\addlinespace
euclid & 42 & 500 & 0 & 150.21 & n/a & -- & 0.2447 & 0.7683 \\
\addlinespace
euclid & 99 & 500 & 0 & 182.50 & n/a & -- & 0.1951 & 0.7867 \\
\addlinespace
euclid & 123 & 500 & 0 & 157.79 & n/a & -- & 0.1932 & 0.3783 \\
\addlinespace
euclid & 777 & 500 & 0 & 138.21 & n/a & -- & 0.1780 & 0.7658 \\
\addlinespace
euclid & 2024 & 500 & 0 & 149.48 & n/a & -- & 0.2816 & 0.7583 \\
\addlinespace
hyperlora & 7 & 500 & 0 & 13.94 & 7.70e-7 & 4.08e4 & 0.1546 & 0.5126 \\
\addlinespace
hyperlora & 42 & 500 & 0 & 17.77 & 1.25e-6 & 1.47e6 & 0.3050 & 0.7611 \\
\addlinespace
hyperlora & 99 & 500 & 0 & 16.51 & 1.13e-6 & 4.49e5 & 0.2253 & 0.6660 \\
\addlinespace
hyperlora & 123 & 500 & 0 & 12.57 & 9.50e-7 & 1.14e4 & 0.2498 & 0.6444 \\
\addlinespace
hyperlora & 777 & 500 & 0 & 25.00 & 1.13e-6 & 1.44e9 & 0.1020 & 0.6075 \\
\addlinespace
hyperlora & 2024 & 500 & 0 & 22.96 & 1.43e-6 & 2.04e8 & 0.3260 & 0.6560 \\
\addlinespace
hysat & 7 & 500 & 0 & 178.01 & 8.90e-7 & -- & 0.2599 & 0.6533 \\
\addlinespace
hysat & 42 & 500 & 0 & 161.80 & 1.25e-6 & -- & 0.1996 & 0.5783 \\
\addlinespace
hysat & 99 & 500 & 0 & 185.15 & 1.37e-6 & -- & 0.2241 & 0.7500 \\
\addlinespace
hysat & 123 & 500 & 0 & 168.40 & 1.37e-6 & -- & 0.2404 & 1.0000 \\
\addlinespace
hysat & 777 & 500 & 0 & 176.26 & 9.50e-7 & -- & 0.1530 & 0.7208 \\
\addlinespace
hysat & 2024 & 500 & 0 & 160.87 & 1.37e-6 & -- & 0.2919 & 0.8333 \\
\bottomrule
\end{tabularx}
\begin{flushleft}\footnotesize
All values are at the final logged step. The unclamped factor is a counterfactual, evaluated at the tangent norm the map would have received had no clamp been applied, and only the adapter manifold arm records it, because only there does an adapter update enter the map; the loss layer arm applies the map to the temperature scaled hidden state, and the Euclidean and cross entropy arms evaluate no map. The tangent norm is the vector entering the map, which is the adapter update in the adapter manifold arm and the projected hidden state in the loss interface arm, so the two columns are not comparable across arms. Neither is defined for the cross entropy arm, which receives no structural supervision and therefore performs no projection, and the departure statistic is not meaningful for the Euclidean arm, which projects to a flat space.
\end{flushleft}
\end{table}

\begin{table}[htbp]
\centering
\caption{\textbf{Extended Data Table 7 |} The intervention ledger of the development lineages, by outcome}
\label{tab:ed7}
\footnotesize\singlespacing
\setlength{\tabcolsep}{4pt}
\begin{tabularx}{\linewidth}{@{}Y r r r r r r r r r@{}}
\toprule
Lineage & Launches & Completed & Segment & Unusable & Failure & Plateau & Collapse flag & Stop or OOM & Unrecorded \\
\midrule
ORAA & 80 & 17 & 4 & 4 & 17 & 1 & 0 & 28 & 9 \\
\addlinespace
MetaTeach & 72 & 17 & 7 & 1 & 9 (+3) & 0 & 0 & 23 & 12 \\
\addlinespace
AdmitBrain & 54 & 22 & 0 & 0 & 5 (+1) & 0 & 0 & 22 & 4 \\
\addlinespace
BS Sovereign & 40 & 20 & 0 & 2 & 1 & 0 & 5 & 7 & 5 \\
\addlinespace
other or unknown code & 6 & 6 & 0 & 0 & 0 & 0 & 0 & 0 & 0 \\
\midrule
\textbf{Total} & \textbf{252} & \textbf{82} & \textbf{11} & \textbf{7} & \textbf{32 (+4)} & \textbf{1} & \textbf{5} & \textbf{80} & \textbf{30} \\
\bottomrule
\end{tabularx}
\begin{flushleft}\footnotesize
Each launch of the four development lineages, and of 6 runs on other code or code whose lineage is not recorded, is one row of the deposited intervention ledger, which gives its date, whether it started fresh or resumed and from which checkpoint, the end step, the outcome, the events of the run, what the record says about the map, the controls the record names, the evidence tier, the reason for any adjudication kept apart from the doubt that remains, and the references of its sources and of the record that ends it: 228 of the 252 rows rest on a record made by the run or on the day of the run or the next. Launches are counted once: a continuation from the launch's own checkpoint without a change of training settings is part of that launch, its interruptions kept as events (8 launches), the same launch described in two records is one row, and a document total or a row bundling several attempts is not a launch. Completed includes 30 short test runs. Failure is a non finite value, an exhausted non finite fallback, a crash of unrecorded type, a finite divergence or a completed run whose representations collapsed; the number in parentheses rests on retrospective records or code comments only. Plateau is a run that ended on a finite plateau; a structural loss term recorded at zero, a clamped loss segment or a loss dominated by a penalty is an event of the launch and does not set its outcome (17 launches carry events). Collapse flag is an interim representation check that read collapse, after which the run stopped or continued; stop or OOM is an operational stop or an out of memory error. Of the 32 failures with a dated record, the record names the exponential map for that launch in 6, gives a cap value without naming the map in 6, places the map only in code saved later the same day in 4, records a transformation or an architecture description in 2, and records the operation as inherited or not at all in 14; 14 of them record an argument cap, an adapter norm cap, an adapter normalisation or a loss clamp. The map is named in the record of the run for 4 completed launches as well, a plan document or a head off the training path being recorded as such in the ledger. The ledger records operating history; it does not assign causes.
\end{flushleft}
\end{table}

\begin{table}[htbp]
\centering
\caption{\textbf{Extended Data Table 8 |} Operating registry of the six systems}
\label{tab:ed8}
\footnotesize\singlespacing
\setlength{\tabcolsep}{4pt}
\begin{tabularx}{\linewidth}{@{}Y Y Y Y r r r Y@{}}
\toprule
System & Base & Output position & Map input & Unique records & Registry exposure & Updates & Evidence \\
\midrule
ORAA & Llama 3.1 8B & Loss interface & projection lifted; no argument formed (stage scripts and logs) & ≥7,454,239 & 11,387,744 & 355,681 & worklogs and logs of all seven training stages \\
\addlinespace
S3 Creative & Flan-T5-XL & Loss interface & pooled state scaled by 0.1, cap 5.0; no trainable factor (code) & 3,320,126 & 4,977,626 & 311,046 & structural code; round files reproduce the three counts \\
\addlinespace
BS Sovereign & from scratch 146M & Loss interface (output projection) & projection, RMS norm bounded by 6.93, cap 5.0 (engine code) & 1,074,084 & 2,148,760 & 33,574 & direct v3 and v4 traces \\
\addlinespace
CEO Finance & EXAONE 3.5 7.8B & Loss interface & projection lifted; no argument formed (code and log) & 1,011,879 & 1,440,619 & 45,017 & production log, all 73 rounds \\
\addlinespace
MetaTeach & EXAONE 3.5 7.8B & Adapter manifold & retraction, rounds 2 to 10 (round logs) & 562,655 & 602,655 & 14,853 & mixed direct windows \\
\addlinespace
AdmitBrain & Llama 3.1 8B & Adapter manifold & retraction (code default) & 364,148 & 394,148 & 12,313 & direct and trainer states \\
\midrule
\textbf{Total} & 3 pretrained families and 1 from scratch & two positions & four constructions & \textbf{≥13,787,131} & \textbf{20,951,552} & \textbf{772,484} & disaggregated \\
\bottomrule
\end{tabularx}
\begin{flushleft}\footnotesize
Unique records are counted within each system in its retained round files, and for the creative question system in its registered source pool, and summed without deduplicating across systems; the research system's count and the total are lower bounds, because 118 of that system's 1,068 round files were not read. Registry exposure counts the input examples of the recorded rounds and epochs, including repeated epochs and curriculum rounds, under the registry's original input and schedule accounting, and is not compared with the unique record count. Versions 1 and 2 of this preprint reported in place of the unique record counts the registry's corpus counts, 17,954,911 in all, which exceed the unique record count for five of the six systems because they count input rows that repeat records within and across rounds and, for the from scratch system, include a held out test split of 90,314 rows (Supplementary Note M10). The teaching system's exposure includes a single later round of 52,273 examples, recorded separately in the Methods and absent from the July 2026 preprint figure. One row of the update total is a lower bound, and the creative question system's row was computed from its round files, its retained records reconstructing fewer updates (Supplementary Note M10). A further 1,744,000 examples entered the largest system at an earlier stage in a different model state and are excluded; including them would give 22,695,552. The placement recorded for the from scratch system differs between the registry and the July 2026 preprint (arXiv:2607.17513v1), which calls it fully hyperbolic: the registered checkpoints, v3 and v4, keep Euclidean blocks and evaluate the geometry at an output projection, while the earlier v2, which carried the geometry through every layer, collapsed and is not in the registry. The base recorded for S3 Creative differs as well: the preprint names T5-3B, while the layer dimensions in its checkpoint's configuration, its stored weight count and the base its training script loads are those of Flan-T5-XL. In the teaching system the pairwise structural term read zero in every retained record of the first two rounds, whose 1,481 step identifiers are part of its update total of 14,853; 13,371 identifiers carry a nonzero value of the term, the records being written during the loss computation and not on completion of an update. Output position and map input are recorded apart because the systems that share the loss interface do not share the computation that feeds the map: the creative question system alone maps a pooled state with no trainable factor before the map, the construction of section 3.1; the research and financial systems lift a learned projection onto the hyperboloid, so no argument of the map is formed; and the from scratch decoder's engine code bounds the norm by 6.93 before its cap of 5.0, the saturation rate of the cap and the binding of that code to the registered checkpoints being unverified. On the adapter manifold the teaching system's round logs record the retraction path from its second round, the branch of its first round being unrecorded, and the admissions system's script defaults to it, the runtime value not being retained. The tier in parentheses is that of the record fixing the map input.
\end{flushleft}
\end{table}

\clearpage
\begin{figure}[H]\centering\includegraphics[width=\linewidth]{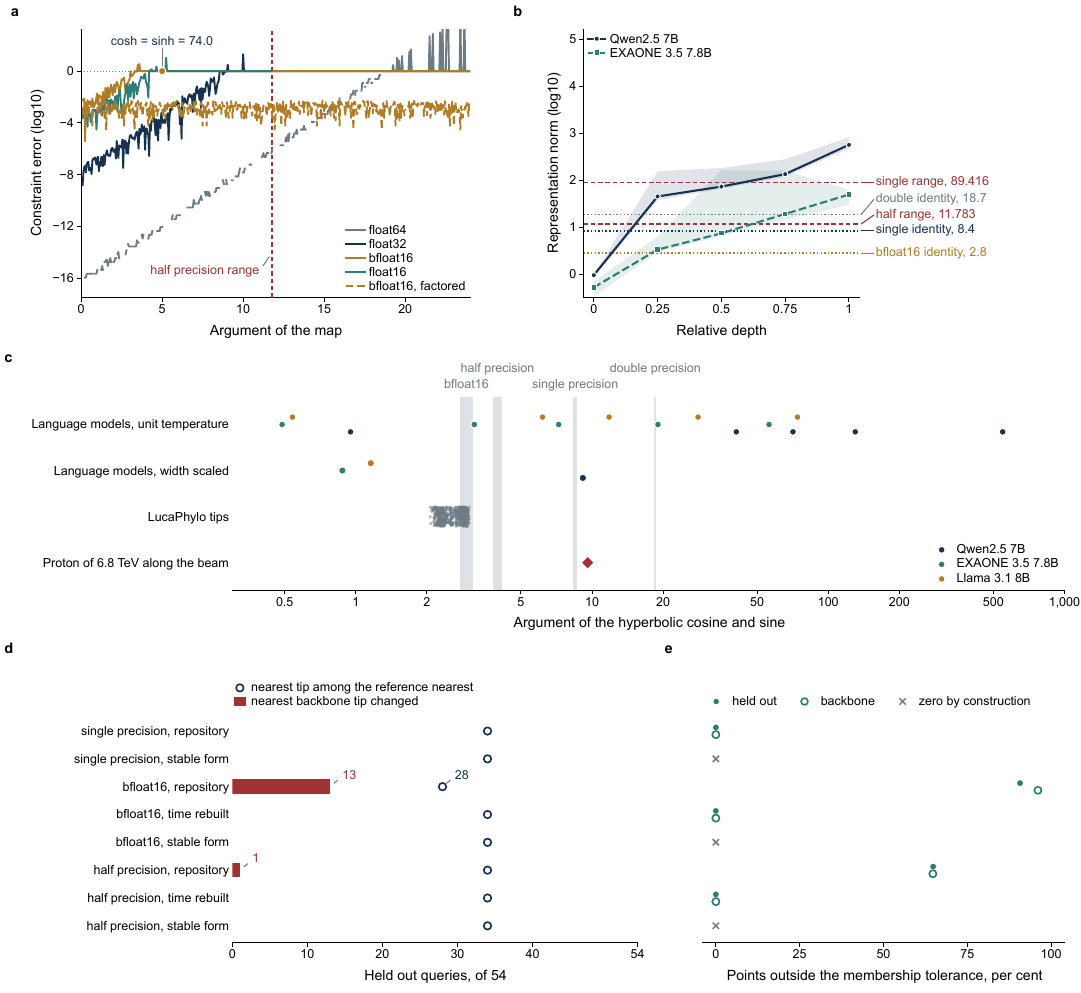}
\caption*{\textbf{Extended Data Fig. 1 \textbar{} The identity and range walls, the representations that reach them, and what bfloat16 storage changed in a phylogenetic model.} \textbf{a}, Lorentz constraint error by storage format: each difference form dies where both coordinates round to one stored number, and half precision stops near 6.2 as its squares overflow first (rule, 11.783); dashed, the constraint factored into exponentials, holding to the exponent wall. \textbf{b}, Hidden norms of two pretrained models by relative depth at every token position (medians, 10th to 90th percentiles; 4,425 records), with the identity (dotted) and range (dashed) walls at unit curvature and temperature. \textbf{c}, Shaded, the coincidence walls of Methods M1, \texttt{p\ ln2\ /\ 2} to \texttt{(p\ +\ 1)\ ln2\ /\ 2} for p significand bits. Rows: median last token hidden norms of three pretrained models at five depths and the deepest divided by the square root of the hidden width (Methods M2); the 536 LucaPhylo tips, as inverse hyperbolic sines of their spatial norms, 160 beyond the bfloat16 wall of 2.811 (Methods M5); the rapidity of a 6.8 TeV proton (section 5). \textbf{d}, \textbf{e}, LucaPhylo coordinates stored in each format and read along the eight paths of Methods M5, fixed by hash before the run. \textbf{d}, Held out queries whose nearest backbone tip changed from the single precision reading (bars) or is among their nearest in the reference tree (circles). \textbf{e}, Points outside the membership tolerance of 0.01; the stable form computes the time coordinate exactly.}\end{figure}

\clearpage
\begin{figure}[H]\centering\includegraphics[width=\linewidth]{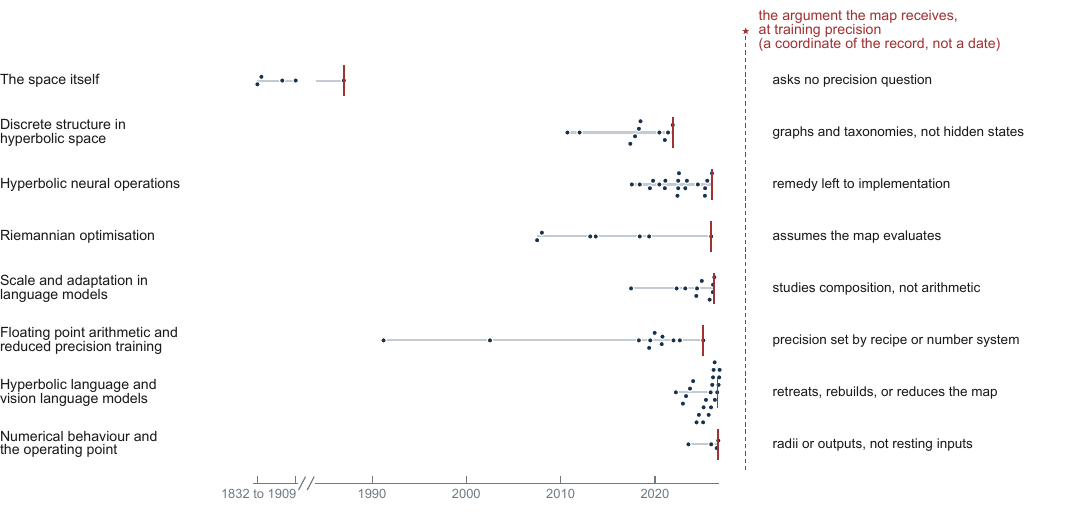}
\caption*{\textbf{Extended Data Fig. 2 \textbar{} What each line of the searched record measured, and the cell none of them holds.} The lineages of the record searched (Methods M13 and M14), one bar each on a shared time axis, one mark per cited work (87 works); works before 1985 sit in the compressed span at left, and coinciding marks are stacked. A vertical tick marks each lineage's last work; the phrase beside it states what that record does and does not measure: the space itself asks no precision question; the embedding of discrete structure treats graphs, taxonomies and cells rather than the hidden states of a trained language model; hyperbolic neural operations bound their inputs and leave the remedy for the arithmetic to implementation; Riemannian optimisation assumes the map evaluates; scale and adaptation in language models study the composition of trainable components, not the arithmetic of evaluation; floating point arithmetic and reduced precision training set precision by recipe or by number system; hyperbolic language and vision language models retreat from, rebuild, or reduce the map; and the numerical studies analyse chosen radii or audit outputs, not the resting inputs of a pretrained model. The dashed line at right is not a date but the cell no line occupies: the argument the exponential map receives from a pretrained model's resting hidden state in the compute format the operation runs in, which this work measures. Each stopping phrase describes the record searched, not a universal negative and not intentions (Methods M13).}\end{figure}
\end{document}